\documentclass[11pt]{article} % For LaTeX2e
\usepackage{fullpage}
\usepackage{natbib}
\usepackage[hyphens]{url}
\usepackage[hyphenbreaks]{breakurl}
\usepackage[colorlinks=true, allcolors=blue]{hyperref}
\usepackage{mathtools}
\usepackage{algorithm}
\usepackage{algorithmic}
\usepackage{amsmath,amssymb,amsthm,amsfonts,bm,bbm, nicefrac}
\usepackage{xcolor}         % colors
\usepackage[a4paper,top=3cm,bottom=2cm,left=3cm,right=3cm,marginparwidth=1.75cm]{geometry}

\definecolor{mydarkblue}{rgb}{0,0.08,0.45}
\definecolor{DarkBlue}{rgb}{0,0,0.55}
\hypersetup{
        breaklinks=true,
	colorlinks=true,     
	linkcolor=blue,
	citecolor=blue, 
	filecolor=blue,      % color of file links
	urlcolor=blue,          % color of external links
	pdftitle={Local Optimization Achieves Global Optimality in Multi-Agent Reinforcement Learning},
        pdfkeywords={multi-agent reinforcement learning, policy optimization, reinforcement learning theory},	
}
\allowdisplaybreaks

\def\rva{{\mathbf{a}}}

\def\va{{\bm{a}}}

\DeclareMathAlphabet{\mathsfit}{\encodingdefault}{\sfdefault}{m}{sl}
\SetMathAlphabet{\mathsfit}{bold}{\encodingdefault}{\sfdefault}{bx}{n}

\def\gA{{\mathcal{A}}}

\def\gC{{\mathcal{C}}}
\def\gD{{\mathcal{D}}}
\def\gE{{\mathcal{E}}}
\def\gF{{\mathcal{F}}}

\def\gL{{\mathcal{L}}}
\def\gM{{\mathcal{M}}}
\def\gN{{\mathcal{N}}}
\def\gO{{\mathcal{O}}}
\def\gP{{\mathcal{P}}}

\def\gS{{\mathcal{S}}}
\def\gT{{\mathcal{T}}}

\def\gX{{\mathcal{X}}}

\def\bgA{{\boldsymbol{\mathcal{A}}}}

\def\sR{{\mathbb{R}}}

\newcommand{\E}{\mathbb{E}}

\DeclareMathOperator*{\argmax}{arg\,max}
\DeclareMathOperator*{\argmin}{arg\,min}

\renewcommand{\algorithmicrequire}{\textbf{Input:}}
\renewcommand{\algorithmicensure}{\textbf{Output:}}

\newenvironment{enumerate*}%
{\begin{enumerate}[leftmargin=*,topsep=0pt]%
		\setlength{\itemsep}{0pt}%
		\setlength{\parskip}{0pt}}%
	{\end{enumerate}}

\theoremstyle{plain}
\newtheorem{theorem}{Theorem}[section]
\newtheorem{proposition}[theorem]{Proposition}
\newtheorem{lemma}[theorem]{Lemma}

\theoremstyle{definition}
\newtheorem{definition}[theorem]{Definition}
\newtheorem{assumption}[theorem]{Assumption}
\theoremstyle{remark}
\newtheorem{remark}[theorem]{Remark}

\def\showcomments{1}
\ifnum\showcomments=1
\newcommand{\simon}[1]{[\textcolor{red}{Simon: #1}]}
\newcommand{\yulai}[1]{[\textcolor{blue}{Yulai: #1}]}
\newcommand{\jianshu}[1]{[\textcolor{purple}{Jianshu: #1}]}
\else
\newcommand{\simon}[1]{}
\newcommand{\yulai}[1]{}
\newcommand{\jianshu}[1]{}
\fi

\def\bpi{{\boldsymbol{\pi}}}

\def\sa{{s, \mathbf{a}^{1:m-1}}}
\def\thetak{{\theta_k}}

\usepackage{tabularx,colortbl,xcolor,booktabs}
\allowdisplaybreaks

\usepackage{graphicx}
\usepackage{amsmath,amssymb,amsthm,amsfonts,enumitem,color}

\title{Local Optimization Achieves Global Optimality in Multi-Agent Reinforcement Learning}

\date{}
\author{
	 Yulai Zhao\thanks{Princeton University; \texttt{yulaiz@princeton.edu}}
	 \and 
	 Zhuoran Yang\thanks{
	 Yale University;
   \texttt{zhuoran.yang@yale.edu}}
\and
Zhaoran Wang\thanks{Northwestern University; 
\texttt{zhaoranwang@gmail.com}}
\and
Jason D. Lee\thanks{Princeton University; 
\texttt{jasonlee@princeton.edu}}
}

\begin{document}

\maketitle
 
\begin{abstract}
Policy optimization methods with function approximation are widely used in multi-agent reinforcement learning. However, it remains elusive how to design such algorithms with statistical guarantees. Leveraging a multi-agent performance difference lemma that characterizes the landscape of multi-agent policy optimization, we find that the localized action value function serves as an ideal descent direction for each local policy. Motivated by the observation, we present a multi-agent PPO algorithm in which the local policy of each agent is updated similarly to vanilla PPO. We prove that with standard regularity conditions on the Markov game and problem-dependent quantities, our algorithm converges to the globally optimal policy at a sublinear rate. We extend our algorithm to the off-policy setting and introduce pessimism to policy evaluation, which aligns with experiments.
To our knowledge, this is the first provably convergent multi-agent PPO algorithm in cooperative Markov games.
\end{abstract}

\section{Introduction}
\label{sec:intro}
Recently, multi-agent reinforcement learning (MARL) has demonstrated many empirical successes, e.g., popular strategy games such as Go~\citep{silver2016mastering}, StarCraft \uppercase\expandafter{\romannumeral2}~\citep{vinyals2019grandmaster}, and poker~\citep{brown2018superhuman}. In contrast to vanilla reinforcement learning (RL), which is only concerned with a single agent seeking to maximize the total reward, MARL studies how multiple agents interact with the shared environment and other agents. 

Policy optimization methods are widely used in MARL. These algorithms often parameterize
policies with a function class and compute the gradients
of the cumulative reward
using the policy gradient theorem~\citep{sutton1999policy} or its variants (e.g., NPG~\cite{kakade2001natural} and PPO~\citep{ppo}) to
update the policy parameters. 

Despite the empirical successes, theoretical studies of policy optimization in MARL are very limited. Even for the cooperative setting where the agents share a common goal: maximizing the total reward function, 
 numerous challenges arise~\citep{zhang2021multi}. (1) non-stationarity: each action taken by one agent affects the total reward and
the transition of state. Consequently,  each learning agent must learn to adapt to the changing environment caused by other agents. From the optimization perspective, the geometry of the multi-agent policy optimization problem becomes unclear. Direct application of traditional single-agent analysis becomes vague due to the lack of stationary Markovian property, which states that evolution in the future only depends on the previous state and individual action. (2) scalability: taking other agents into consideration, each individual agent would face the joint action
space, whose dimension increases exponentially with the number of agents. Thus, having numerous agents in the environment problematizes the theoretical analysis of MARL. (3) function approximation: closely related to the scalability issue, the state space and joint action space are often immense in MARL, promoting function approximation to become a necessary component in MARL at the ease of computation and statistical analysis.

In this paper, we aim to answer the following fundamental
question:
\begin{center}
	\emph{Can we design a provably convergent multi-agent policy optimization algorithm in the cooperative setting with function approximation?}
\end{center}

We answer the above question affirmatively. We propose a multi-agent PPO algorithm in which the local policy of each agent is updated \textbf{sequantially} in a similar fashion as vanilla PPO algorithm~\citep{ppo}. In particular, we leverage a multi-agent performance difference lemma (cf. Lemma~\ref{lem:perf}), assuming the joint policy is decomposed into conditional dependent policies. Such a lemma characterizes the landscape of policy optimization, showing the superiority of using localized action value functions as the decent direction for each local policy. Such factorized structure essentially bypasses the non-stationarity and scalability concerns.
To address large state spaces, we parameterize each local policy using log-linear parametrization and propose to update the policy parameters via KL divergence-regularized mirror descent, where the descent direction is estimated separately. Combining these results, we obtain our multi-agent PPO algorithm.
We prove that the multi-agent PPO algorithm converges to globally optimal policy at a sublinear rate. 
Furthermore, we extend multi-agent PPO to the off-policy setting in which policy is evaluated using samples collected according to \emph{data distribution} $\mu$.
% in which pessimism is used to adjust the value computation bias from  the use of function approximators.
% lower confidence bounds of the critic functions in the policy updates. 
We prove similar theoretical guarantees under a coverage assumption of the sampling distribution.

We summarize our contributions below.
\paragraph{Our contributions.}
First, by focusing on the factorized policies, we prove a multi-agent version of the performance difference lemma showing that the action value functions are ideal descent directions for local policies. Such a geometric characterization functions as a remedy for the non-stationarity concern, motivating our multi-agent PPO algorithm.

Second, we adopt the log-linear function approximation for the policies. We prove that multi-agent PPO converges at a sublinear $\gO\left(\frac{N}{1-\gamma} \sqrt{\frac{\log|\gA|}{K}} \right)$ rate up to some statistical errors incurred in evaluating/improving policies, where $K$ is the number of iterations, $N$ is the number of agents and $|\gA|$ is the action space of each individual agent. The sample complexity depends polynomially on $N$, thus breaking the curse of scalability. 

Third, we propose an off-policy variant of the multi-agent PPO algorithm and introduce pessimism into policy evaluation. The algorithm also converges sublinearly to the globally optimal policy up to the statistical error $\widetilde{\gO}(n^{-\frac{1}{3}})$. Here, $n$ is the number of samples used to estimate the critics.\footnote{
$\widetilde{\gO}\left(\cdot\right)$ hides logarithmic factors.
}A key feature of the sample complexity bound is that it only requires single-policy concentrability.
% Our algorithm has the similar structure as in previous work on policy-based RL algorithms~\citep{agarwal2020optimality}. In each iteration, we first evaluate the quality of the current policy in the \textbf{policy evaluation} step, then leverage recent developments on PPO in single-agent RL~\citep{liu2019neural} to update policy in the \textbf{policy improvement} step. Technically, by combining the mirror-descent analysis~\citep{agarwal2020optimality,liu2019neural} with the multi-agent dependency characteristics,
% we prove an $\gO\left(\nicefrac{1}{\sqrt{K}}\right)$ rate where $K$ is the number of iterations. 

To our knowledge, this is the first provably convergent multi-agent PPO algorithm in cooperative Markov games with function approximation.

\paragraph{Organization.}
This paper is organized as follows.
In Section~\ref{sec:related_work}, we review related literature.
In Section~\ref{sec:marl}, we formally describe the problem setup and introduce the necessary definitions.
In Section~\ref{sec:main_thm}, we state the main multi-agent PPO algorithm in detail. We further extend our results to the off-policy setting in Section~\ref{sec:pess}. 
We conclude in Section~\ref{sec:con} and defer the proofs to the Appendix.

\section{Related Work}
\label{sec:related_work}
\paragraph{Policy optimization}
Many empirical works have proven the validity and efficiency of policy optimization methods in games and other applications~\citep{silver2016mastering, silver2017mastering, guo2016deep, tian2019elf}.
% Below, we primarily focus on relevant algorithmic and theoretical papers.
% The current paper focuses on using NPG techniques for solving two-player zero-sum Markov games. 
% NPG is first introduced by \citet{kakade2002natural} to explore better the underlying structure of the reinforcement learning (RL) problem instance.
% Extensions of NPG methods are also used to solve zero-sum games. \citet{zhang2019policy,bu2019global} applied projected natural nested gradient under a linear quadratic setting, a significant class of zero-sum Markov games.
% Extensions to imitation learning were also studied in~\citep{song2018multi}.
These works usually update the policy parameter in its parametric space using the pioneering policy gradient (PG) theorem by~\citet{sutton1999policy}, or many PG variants invented to improve the empirical performances of vanilla PG methods.
 In particular, ~\citet{kakade2001natural} introduced the natural policy gradient (NPG) algorithm which searched for the steepest descent direction within the parameter space based on the idea of
KL divergence-regularization. 
% In contrast, we provide bounds for the two-player zero-sum case, which is significantly more challenging.
% In two-player zero-sum games, the non-stationary environment faced by each agent invalidates the stationary structure of the single-agent setting and thus precludes the direct application of the convergence proof from the single-agent setting.
% Furthermore, each agent in two-player zero-sum games must adapt to the other agent's policy, which poses additional difficulties. 
% 
% \citet{zhang2020global} proposed a new variant of PG methods that yielded unbiased estimates of policy gradients, which enabled non-convex optimization tools to be applied in establishing global convergence. Despite being non-convex, \citet{agarwal2020optimality,bhandari2019global} identified structural properties of finite MDPs: the objective function has no suboptimal local minimum. They further gave conditions under which any local minimum is near-optimal. 
Trust region learning-based algorithms are often regarded as advanced policy optimization methods in practice~\citep{lillicrap2015continuous, duan2016benchmarking}, showing superior performances with stable updates. Specifically, TRPO~\citep{trpo} and PPO~\citep{ppo} could be seen as KL divergence-constrained variants of NPG. A benign feature of these algorithms is the monotonic improvement guarantees of the expected return. 

Despite prosperous empirical findings, the lack of convexity often impedes the development of theories for policy optimization methods. Denote $K$ and $T$ as the number of iterations and samples. ~\citet{agarwal2020optimality} showed an  iteration complexity of $\gO(K^{-\frac{1}{2}})$ and a sample complexity of $\gO(T^{-\frac{1}{4}})$ for online NPG with function approximation.  ~\citet{shani2020adaptive} considered a sample-based TRPO and proved a $\tilde{\gO}(T^{-\frac{1}{2}})$ rate converging to the global optimum, which could be improved to $\tilde{\gO}(\nicefrac{1}{T})$ when regularized.
 Making minor modifications to the vanilla PPO algorithm, ~\citet{liu2019neural} presented a convergence rate of $\gO(K^{-\frac{1}{2}})$ to global optima when parameterizing both policy and $Q$ functions with neural networks. The key to their analysis is the desirable one-point monotonicity in infinite-dimensional mirror descent that assists in characterizing the policy updates without convexity. We also make use of similar one-point properties in our multi-agent PPO algorithm analysis. 
% Applying NPG to linear quadratic games, \citet{zhang2019policy} and \citet{bu2019global} proved that: for finding Nash equilibrium, NPG enjoys a sublinear convergence rate. Both analyses rely on the linearity of the dynamics, which does not hold in general Markov games considered in this paper.
% (e.g., TRPO~\citep{trpo} and PPO~\citep{ppo})

\paragraph{MARL} 
% The complex inherent nature renders MARL many categories, e.g., cooperative/competitive, where the agents cooperate/compete with each other, 
Markov Game (MG) is a commonly used model to characterize the multi-agent decision-making process~\citep{shapley1953stochastic,littman1994markov}, which can be regarded as a multi-agent extension to the Markov Decision Process (MDP).
% There is a long line of work developing computationally efficient algorithms for cooperative/competitive RL in MGs. 
% The existing algorithms providing non-asymptotic guarantees could be classified into two main types: value-based and policy-based. Value-based approaches~\citep{shapley1953stochastic, patek1997stochastic, littman1994markov, perolat2015approximate, bai2020provable, bai2020near,xie2020learning} seek to find the optimal value function by generalizing its single-agent counterparts. A direct drawback of such an adaptation is the curse of dimensionality due to the combinatorial structure of multi-agent systems. 
Policy-based algorithms could generalize to large states through function approximation. There has been growing interest in developing provably efficient algorithms for Markov games~\citep{daskalakis2020independent,cen2021fast,zhao2022provably, ding2022independent,cen2022faster}. These works often studied competitive RL settings, e.g., zero-sum games. Their convergence
rates usually depended on various notions of concentrability coefficient and may not scale tightly under the worst scenario. 

\paragraph{Policy optimization for MARL} Applying policy optimization methods in the MARL setting is more complicated than in the single-agent setting because of the non-stationary environment faced by each agent~\citep{zhang2021multi}. A learning paradigm called centralized
training with decentralized execution (CTDE) is often used in practice~\citep{kraemer2016multi, lowe2017multi, foerster2018counterfactual,yang2018mean, wen2019probabilistic, zhang2020bi}. In CTDE, a joint centralized value function helps to address the non-stationarity issue caused by other agents. Each agent has access to the global state and actions of other agents during training, thus allowing them to adjust their policy parameters individually. For instance, ~\citet{lowe2017multi} proposed a multi-agent policy gradient algorithm in which agents learned a centralized critic based on the observations and actions of all agents.

Trust region learning~\citep{trpo} has recently been combined with the CTDE paradigm to ensure monotonic improvements. In particular, IPPO~\citep{de2020independent} and MAPPO~\citep{yu2021surprising} showed strong performances of PPO-based methods in the cooperative setting. The practical efficacy of these methods is usually restricted by the \emph{homogeneity} assumption, where the agents share a common action space and policy parameter.
Theoretically, providing statistical guarantees for policy optimization algorithms in MARL is more complicated than single-agent scenario~\citep{zhang2021multi}. In Markov games, the non-stationary environment faced by each agent precludes direct application of the single-agent convergence analysis.
A recent attempt by~\citet{kuba2022trust} proposed the first set of trust region learning algorithms in MARL that enjoyed monotonic improvement guarantees assuming neither homogeneity of agents nor value function decomposition rule. The critical observation leading to their results is the multi-agent advantage function decomposition rule that builds the sequential policy update structure. 
% However, they did not show the theoretical rate towards convergence. This work presents algorithms that converge globally optimal at sublinear rates using the multi-agent advantage decomposition lemma.
% The most relevant paper in the literature is ~\citet{kuba2022trust}.
However, they did not show rates of convergence. In this work, we design a new, provably convergent PPO algorithm for fully cooperative Markov games that converges to globally optimal at policy at sublinear rates by taking advantage of this conditional dependency structure.

\paragraph{Pessimism-based RL methods} Though being able to account for large state/action spaces, function approximation also has its own drawbacks. A significant issue arising in using function approximators is the usual occurrence of a positive bias in value function~\cite{thrun1993issues}. The learner may not receive an accurate assessment. Numerous empirical works leverage the principle of \emph{pessimism} to correct such overestimation~\citep{fujimoto2018addressing,laskin2020reinforcement,lee2020predictive,moskovitz2021tactical}. For example, to reduce the evaluation bias brought by function approximation, ~\citet{fujimoto2018addressing} constructed  the Bellman target by choosing the minimum of two value estimates as an intuitive estimate lower bound. Their approach took a pessimistic view of the value function.

On the theoretical side, a growing body of literature in offline reinforcement learning has also focused pessimism to account for datasets lacking data coverage~\citep{liu2020provably,jin2021pessimism, uehara2021pessimistic,rashidinejad2021bridging,zhan2022offline}. Technically, these works aimed at maximizing the worst-case rewards that a trained agent could obtain. Instead of relying on coverage assumptions on dataset~\citep{munos2003,munos2008finite,chen2019information}, these methods provided dataset-dependent performance bounds, thus providing robust results for datasets lacking exploration for which traditional methods do not apply. We focus on the off-policy setting in Section~\ref{sec:pess} where we leverage the Bellman-consistent pessimism~\citep{xie2021bellman}. We show concrete bounds under linear function approximation by assuming a sampling oracle that provides rewards and transition estimates that are used in approximating action value functions.

\section{Preliminaries}
\label{sec:marl}
In this section, we introduce necessary notations, problem setup, and some useful quantities that will be frequently used in this work.

\subsection{Setup and Notations}
\paragraph{Setup} We consider a \emph{fully-cooperative} Markov game \citep{shapley1953stochastic,littman1994markov}, which is defined by a tuple $\left( \gN, \gS, \bgA, \gP, r, \gamma \right)$. Here, $\gN= \{1, \dots, N\}$ denotes the set of agents, $\gS$ is the finite state space, $\bgA = \gA^N$ is the product of finite action spaces of all agents(i.e., joint action space), $\gP:\gS \times \bgA \times \gS \to [0, 1]$ decides the transition scheme, a reward function $r:\gS \times \bgA \to [0, 1]$, and $\gamma\in[0, 1)$ is the discount factor.\footnote{
For clarity, we assume $N$ agents share the same set of actions. It is straightforward to generalize our results to the setting where action sets are different. See Section~\ref{sec:main_thm}.
}
The agents interact with the environment according to the following protocol: at time step $t$, the agents are at state $s_t \in \gS$; every agent $i$ takes action $a^i_{t} \in \gA$, drawn from its policy $\pi^{i}(\cdot| s_t)$, which together with actions of other agents gives a joint action $\rva_t = (a^{1}_{t}, \dots, a^{N}_t) \in \bgA$, drawn from the joint policy $\bpi(\cdot|s_t)= \prod_{i=1}^{N}\pi^{i}(\cdot|s_t)$; the agents receive a joint reward $r_{t}=r(s_{t}, \rva_{t})\in \sR$, and move to $s_{t+1} \sim \gP(\cdot |s_t, \rva_{t})$. Given the joint policy $\bpi$, the transition probability function $\gP$, 
and the initial state distribution $\rho$, 
we define the discounted
occupancy state-action distribution as
$$
d_\bpi(s,\rva) = (1-\gamma)\E \sum_{t=0}^\infty \text{Pr}^\bpi(s_t = s, a_t = \rva|s_0 \sim \rho).
$$
The standard value function and action value function are defined as
\begin{gather*}
    V_{\bpi}(s) \triangleq        \mathop{\E}\limits_{\rva_{0:\infty}\sim\bpi, s_{1:\infty}\sim \gP}\left[ \sum_{t=0}^{\infty}\gamma^{t} r_t
\Big| \ s_{0} = s \right], \\ Q_{\bpi}(s, \va)  \triangleq \mathop{\E}\limits_{s_{1:\infty}\sim \gP, \rva_{1:\infty}\sim\bpi}\left[ \sum_{t=0}^{\infty}\gamma^{t} r_t 
    \Big| \ s_{0} = s,  \ \rva_{0} = \va \right].
\end{gather*}

% We make the following regularity conditions on the reward function.
% \begin{assumption}[Bounded Reward] There exists a positive constant $H$ such that $H = \sup_{(s,\rva) \in \gS \times \bgA  } |r(s, \rva)|$, which implies $|V_\bpi(s)|, |Q_\bpi(s, \rva)| \le \frac{H}{1-\gamma}$ for any joint policy $\bpi$.
% \end{assumption}

 The standard advantage function considering all agents is written as 
$A_{\bpi}(s, \va) \triangleq Q_{\bpi}(s, \va) - V_{\bpi}(s)$. Later, we shall introduce the agents-specific advantage functions.

Let $\nu_\bpi(s)$ and $\sigma_\bpi(s, \rva) = \bpi(\rva | s) \cdot \nu_\bpi(s)$ denote the stationary state distribution and the
stationary state-action distribution associated with a joint policy $\bpi$, respectively. Define the underlying optimal policy as $\bpi_*$. We use $\nu_*$ and $ \sigma_*$ in this paper to indicate $\nu_{\bpi_*}$ and $\sigma_{\bpi_*}$ for simplicity.

Throughout this paper, we pay close attention to the contribution of different subsets of agents to the performance of the whole team. 
We introduce the following multi-agent notations before proceeding to multi-agent definitions. 

\paragraph{ Notations} In this work, we index the $N$ agents with integers from $1$ to $N$ and use set $\gN = \{i |i = 1, \cdots, N\}$ to represent all agents. We use $m \in \gN$ to indicate the specific $m$-th agent. In particular, the set notation on the superscript of a term represents the quantities associated with agents in that set. For example, $\rva^{\{1,2,3\}}$ represents the joint action of agents $1,2$ and $3$. We may write index $k$ on superscript when we refer to the specific $k$-th agent. When bold symbols are used without any superscript (e.g., $\rva$), they consider all agents. For simplicity, let $(m:m^\prime)$ be shorthand for set: $\{i|m\le i \le m^\prime, i \in \gN \}$. An example is $\bpi^{1:m}(\cdot|s)$ which represents the joint policy considering agents $1,2\cdots, m$.

We now introduce the multi-agent action value functions and advantage functions that characterize contributions from specific sub-agents.
\begin{definition}
\label{def::1}
Let $P$ be a subset in $\gN$. The multi-agent action value function associated with agents in $P$ is 
\begin{align*}
    Q_{\bpi}^{P} \left(s, \rva^{P}\right) \triangleq \E_{\Tilde{\rva}\sim \Tilde{\bpi}}\left[ Q_{\bpi}\left(s, \rva^{P}, \Tilde{\rva}\right) \right],
\end{align*}
here we use a tilde over symbols to refer to the complement agents, namely $\tilde{\rva} = \{a^i|i \not\in P, i \in \gN \}$.

Let $P, P^\prime \subseteq \gN$ be two disjoint subsets of agents. The multi-agent advantage function is defined below. Essentially, it accounts for the improvements of setting agents $\rva^{P^\prime}$ upon setting agents $\rva^{P}$, while all other agents follow $\bpi$. 
    \begin{equation*}
        A_{\bpi}^{P^\prime}\left(s, \rva^{P}, \rva^{P^\prime} \right) \triangleq  Q_{\bpi}^{P \cup P^\prime}\left( s, \rva^{P}, \rva^{P^\prime}\right) 
        - Q_{\bpi}^{P}\left( s, \rva^{P}\right).
    \end{equation*}
\end{definition} 

% Below, we show a crucial lemma by~\citet{kuba2022trust}. It shows that the joint advantage function can be decomposed into a summation of local advantages from sequential updates. Intuitively, our algorithms build on this sequential dependency scheme.

% Hereafter, the joint policies $\bpi = (\pi^1, \dots, \pi^N)$ and $\bar{\bpi} = (\bar{\pi}^1, \dots, \bar{\pi}^N)$ shall be thought of as the ``current", and the ``new" joint policy that agents update towards, respectively.
The multi-agent Bellman operators are defined by generalizing the classic versions.

\begin{definition} For $m \in \gN$ and any function $f: \gS \times \gA^m \xrightarrow{} \sR$ we define \emph{multi-agent Bellman operator} $\gT^{1:m}_\bpi: \sR^{\gS \times \gA^m} \mapsto \sR^{\gS \times \gA^m}$ as 
\begin{align*}
     \gT^{1:m}_\bpi & f(s, \rva^{1:m})
    \coloneqq \mathop{\E}\limits_{\Tilde{\rva} \sim \Tilde{\bpi}} r(s, \rva^{1:m}, \tilde{\rva}) + \gamma \mathop{\E}\limits_{ \substack{\tilde{\rva} \sim \tilde{\bpi} \\
s^\prime \sim \gP(\cdot|s, \rva^{1:m}, \tilde{\rva}) } } f(s^\prime, \bpi^{1:m})
\end{align*}
where $f(s^\prime, \bpi^{1:m})$ is shorthand for $\E_{\rva^\prime \sim \bpi^{1:m}(\cdot|s^\prime)} f(s^\prime, \rva^\prime)$.
\end{definition}
It is straightforward to see that  $Q_{\bpi}^{1:m}$ is the unique fixed point for $\gT_\bpi^{1:m}$, which corresponds to the classic single-agent Bellman operator.

\subsection{KL divergence-regularized mirror descent} 
% Our provable multi-agent PPO algorithm is closely related to the classic single-agent version.  The difference is that we are able to sequentially optimize local improvements by taking advantage of Lemma~\ref{lem:decom}.

We review the mirror-decent formulation in provable single-agent PPO algorithm~\citep{liu2019neural}. At the $k$-th iteration, the policy parameter $\theta$ is updated via
\begin{align} \label{ppo}
     \theta_{k+1} &\xleftarrow{} \argmax_\theta  \hat{\E} \big[ \left \langle A_k(s, \cdot), \pi_\theta(\cdot|s) \right \rangle - \beta_k KL\left(\pi_\theta(\cdot\|s)\|\pi_{\theta_k}(\cdot|s)\right) \big]. 
\end{align}
Hereafter we shall use $\langle \cdot, \cdot \rangle$ to represent the inner product over $\gA$. The expectation is taken over $\hat{\E}$, which is an empirical estimate of stationary state-action distribution $\nu_{\pi_{\theta_k}}$, and $A_k$ is  estimate of advantage function $A^{\pi_{\theta_k}}$.

Adopting the KL-divergence, ~\eqref{ppo} is closely related to the NPG~\citep{kakade2001natural} update. As a variant, this formulation is slightly different from the vanilla PPO~\citep{ppo}: here $KL(\pi_\theta(\cdot\|s)\|\pi_{\theta_k}(\cdot|s))$ is used instead of $KL(\pi_{\theta_k}(\cdot\|s)\|\pi_{\theta}(\cdot|s))$. Such variation is essential for presenting provable guarantees, which will be shown in the next section.

\section{Multi-Agent PPO}
\label{sec:main_thm}
Recall that $\nu_*$ is the stationary state distribution for $\bpi_*$. In this section, we desire to maximize the expected value function under distribution $\nu_*$: $J(\bpi) \triangleq \E_{s \sim \nu_*}  V_{\bpi}(s)$.

This paper aims to present a trust region learning multi-agent algorithm that enjoys a rigorous convergence theory. 
As we have mentioned, policy optimization for cooperative MARL is challenging because the policy optimization problem becomes a joint optimization involving all the agents. It remains unclear: (a) what the landscape of the total rewards as a multivariate function of the joint policy is and (b) what would be proper policy descent directions for each agent. We come up with a solution to characterize the landscape by taking advantage of a serial decomposition of the performance difference lemma in MARL described below.

\begin{lemma} 
\label{lem:perf}
For any joint policy $\bpi$ we have
\begin{align*}
J(\bpi_*) - J(\bpi) 
= \frac{1}{1-\gamma} \sum_{m=1}^N \mathop{\E}\limits_{\substack{s \sim \nu_* \\ \rva^{1:m-1} \sim \bpi_{*}^{1:m-1}}}\left\langle Q_{\bpi}^{1:m} (\sa, \cdot), \pi_*^m(\cdot|\sa) - \pi^m(\cdot|s, \rva^{1:m-1}) \right\rangle
\end{align*}
where the inner product is over $a^m \in \gA$.
\end{lemma}

With this geometric characterization, we can justify that using $Q_{\bpi}^{m}(\sa, a^m)$ as the descent direction and running KL divergence-regularized mirror descent for \textbf{each} agent $m \in \gN$ can lead to a better \textbf{total} reward, which enables a serial optimization procedure. 
Below we describe the algorithm in detail.

To represent conditional policies, we adopt log-linear parametrization.

\paragraph{Parametrization} For the $m$-th agent ($ m \in \gN$), its conditional policy depends on all prior ordered agents $\rva^{1:m-1}$. Given a coefficient vector $\theta^m \in \Theta$, where $\Theta = \{\|\theta\| \le R | \theta \in \sR^d  \}$ is a convex, norm-constrained set. The probability of choosing action $a^m$ under state $s$ is
\begin{equation} \label{eq:loglinear}
    \pi_{\theta^m} (a^m|\sa)= \frac{\exp{( \phi^\top(\sa, a^m) \theta^m)}}{\mathop{\sum}\limits_{a^m \in \gA} \exp{( \phi^\top(\sa, a^m) \theta^m)}}
\end{equation}
where $\phi$ is a set of feature vector representations. Without loss of generality, we impose a regularity condition such that every $\|\phi\|_2 \le 1$. This parametrization has been widely used in RL literature~\citep{branavan2009reinforcement, gimpel2010softmax, heess2013actor,agarwal2020optimality,zhao2022provably}.\footnote{We assume that all players share the same parameter set only for clarity. We only need minor modifications in the analysis to extend our results to the setting where $N$ agents have different capabilities. Specifically, we only need to treat norm bounds of updates ($R$), regularity conditions on features, and $\beta$ separately for each agent.}

\subsection{Policy Improvement and Evaluation}
At the $k$-th iteration, we have the current policy $\bpi_\thetak$, and we need to: (1) perform \textbf{policy evaluation} to obtain the action value function estimates  $\hat{Q}_{\bpi_\thetak}$ for determining the quality of $\bpi_\thetak$. (2) perform \textbf{policy improvement} to update policy to $\bpi_{\theta_{k+1}}$.

For notational simplicity, we use $\nu_k$ and $\sigma_k$ to represent stationary state
distribution $\nu_{\bpi_\theta^k}$
and the stationary state-action distribution $\sigma_{\bpi_\theta^k}$, which are induced by $\bpi_\thetak$. 

%  This section studies function approximation to generalize across a large state space.

% The pseudo code of Algorithm~\ref{alg:2} is shown below. 
% \begin{algorithm}
% \caption{Multi-Agent TRPO}
% \label{alg:2}
% \begin{algorithmic}[1]
% \STATE Initialise the joint policy $\boldsymbol{\pi}_{0} = (\pi^{1}_{0}, \dots, \pi^{n}_{0})$.
% \FOR{$k=0, 1, \dots$}
%     \FOR{$m=1:n$}
%         \STATE 
%         \begin{align*}
%          \pi^{m}_{k+1}(\cdot|s, \rva^{1:m-1})
%         &= \argmax_{\pi^{m}(\cdot|s, \rva^{1:m-1})} \E_{\nu_k} \left \langle \pi^m(\cdot|s, \rva^{1:m-1}), A_{\boldsymbol{\pi}_k}^{m}(s, \rva^{1:m-1}, \cdot) \right \rangle \\
%         &\qquad \qquad -\beta KL\left(\pi^{m}(\cdot|s, \rva^{1:m-1})\| \pi^{m}_{k}(\cdot|s, \rva^{1:m-1})\right)
%         \end{align*}
%     \ENDFOR
% \ENDFOR
% \end{algorithmic}
% \end{algorithm}

% \begin{align*}
%          \pi^{m}_{k+1}(\cdot|s, \rva^{1:m-1})
%         &= \argmax_{\pi^{m}(\cdot|s, \rva^{1:m-1})} \E_{\nu_k} \left \langle \pi^m(\cdot|s, \rva^{1:m-1}), A_{\boldsymbol{\pi}_k}^{m}(s, \rva^{1:m-1}, \cdot) \right \rangle \\
%         &\qquad \qquad -\beta KL\left(\pi^{m}(\cdot|s, \rva^{1:m-1})\| \pi^{m}_{k}(\cdot|s, \rva^{1:m-1})\right)
%         \end{align*}

\paragraph{Policy Improvement} At the $k$-th iteration, we define $\hat{\pi}_{k+1}^m$ as the ideal update based on $\hat{Q}_{\bpi_\thetak}^{1:m}$ (for agent $m \in \gN$), which is an estimator of $Q_{\bpi_\thetak}^{1:m}$. The ideal update is obtained via the following update
\begin{gather}
\label{eq:ideal_upd}
    \hat{\pi}_{k+1}^m \xleftarrow{} \argmax_{\pi^m} \hat{F}(\pi^m) \\
     \hat{F}(\pi^m) =  \mathop{\E}\limits_{\sigma_k}  \Big[ \langle \hat{Q}_{\bpi_\thetak}^{1:m}(s, \rva^{1:m-1}, \cdot), \pi^m (\cdot|s, \rva^{1:m-1})\rangle  -\beta_k KL\left(\pi^m(\cdot|s, \rva^{1:m-1})\| \pi_{\theta_k^m}(\cdot|s, \rva^{1:m-1})\right) \Big] \notag
\end{gather}
where $\theta_k^m$ is the parameter of the current conditional policy of the $m$-th agent. In above equation, the distribution is taken over $(\sa) \sim \nu_k  \bpi_{\thetak}^{1:m-1}$, we write $\sigma_k$ for simplicity. Under log-linear parametrization: $\pi_{\theta_k^m} \propto \exp \{ \phi^\top \theta_k^m\}$, we have the following closed-form ideal policy update.

\begin{proposition}
\label{prop:update}
Given an estimator $\hat{Q}_{\bpi_\thetak}^{1:m}$, 
% Let $\pi_{k}$ be an energy-based policy, i.e., $\pi_{k}(\cdot|s, \rva^{1:m-1}) \propto \exp \pi^m{\theta_k^\top \phi(s, \rva^{1:m-1}, \cdot)\}$, 
the KL divergence-regularized update~\eqref{eq:ideal_upd} has the following explicit solution
\begin{align*}
     \hat{\pi}^m_{k+1} (\cdot| s, \rva^{1:m-1}) \propto \exp \left\{ \beta_k^{-1} \hat{Q}_{\bpi_\thetak}^{1:m}(s, \rva^{1:m-1}, \cdot)\
    + \phi^\top(s, \rva^{1:m-1}, \cdot) \theta_k^m  \right\}.
\end{align*}
\end{proposition}

The proof is straightforward by adding the constraint: $\sum_{a^m \in \gA} \pi^m(\cdot) = 1 $ as a Lagrangian multiplier to $\hat{F}(\pi^m)$. See details in Appendix~\ref{sec:pf-ma}.

  To approximate the ideal $\hat{\pi}^m_{k+1}$ using a parameterized $\pi_{\theta_{k+1}^m} \propto \exp \{ \phi^\top \theta_{k+1}^m\}$, we minimize the following mean-squared error (MSE) as a sub-problem
\begin{align} \label{eq:MSE}
    & \theta_{k+1}^m \xleftarrow{}
     \argmin_{\theta^m \in \Theta} L(\theta^m)
\end{align}
where $L(\theta^m)$ is defined as
$$
L(\theta^m) = \mathop{\E}\limits_{\sigma_k} \Big( (\theta^m-\theta_k^m)^\top \phi(s, \rva^{1:m-1}, a^m) -   \frac{\hat{Q}_{\bpi_\thetak}^{1:m}(s, \rva^{1:m-1}, a^m)}{\beta_k}  \Big)^2
$$
 % We mention that in Eq~\eqref{eq:MSE} we sample actions from $\bpi$ to guarantee that $\pi_{\theta_{k+1}}^m$ approximates the ideal infinite-dimensional policy $\hat{\pi}_{k+1}^m$ evenly over all actions. By fixing the distribution for sampling actions, we can apply off-policy sampling of states and actions.
Intuitively, a small $L(\theta)$ indicates that $\pi_{\theta^m}$ is close to the ideal update $\hat{\pi}_{k+1}^m$. Moreover, if $\hat{\pi}_{k+1}^m$ exactly lies in the log-linear function class, i.e., there exists a $\vartheta \in \Theta$ such that $\hat{\pi}_{k+1}^m \propto \exp{} \{ \phi^\top \vartheta\}$. Then we have $L(\vartheta) = 0$.

To solve the MSE minimization problem~\eqref{eq:MSE}, we use the classic SGD updates. Let stepsize be $\eta$, at each step $t = 0,1, \cdots, T-1$, parameter $\theta$ is updated via
\begin{align*}
     \theta(t+\frac{1}{2}) &\xleftarrow{} \theta(t)- 2\eta \phi \left( (\theta(t) - \theta_k^m)^\top \phi   -\beta_k^{-1} \hat{Q}_{\bpi_\thetak}^{1:m}) \right)\\
    \theta(t+1) &\xleftarrow{} \Pi_{\Theta} \theta(t + \frac{1}{2})
\end{align*}
where we omit $(\sa, a^m)$ for simplicity, which is sampled from $\sigma_k$.  See Algorithm~\ref{alg:3} for the detailed solver.

\paragraph{Policy Evaluation} In this step, we aim to examine the quality of the attained policy. Thereby, a $Q$-function estimator is required. We make the following assumption.

\begin{assumption}
\label{assump:estimator}
Assume we can access an estimator of $Q$ function that returns $\hat{Q}$. The returned $\hat{Q}$ satisfies the following condition for all $m \in \gN$ at the $k$-th iteration
\begin{gather*}
\left[ \mathop{\E}\limits_{\sigma_k}   \left( \hat{Q}_{\bpi_\thetak}^{1:m}(\sa, a^m) - Q_{\bpi_\thetak}^{1:m}(\sa, a^m) \right)^2 \right]^{1/2}  \le \xi_k^m.
\end{gather*}
We also have a regularity condition for the estimator: there exists a positive constant $B$, such that 
for any $m \in \gN$ and $(\sa, a^m) \in \gS \times \gA^{m-1} \times \gA$,
$$
\left| \hat{Q}_{\bpi_\thetak}^{1:m}(s, \rva^{1:m-1}, a^m) \right| \le B.
$$
\end{assumption}

In RL practice, such an estimator is often instantiated with deep neural networks (DNNs)~\citep{mnih2015human}. While there has been recent interest in studying the theoretical guarantees for DNNs as function approximators~\citep{fan2020theoretical}, we assume we have access to such an estimator to ensure the generality of our algorithm.
We note that policy estimators like episodic sampling oracle that rolls out trajectories~\citep{agarwal2020optimality} or neural networks~\citep{mnih2015human,liu2019neural} could all be possible options here. As a generalization, we introduce a specific value function approximation setting in Section~\ref{sec:pess}, in which we assume all $Q$-functions lie in linear class $\gF$. We further adopt the principle of pessimism for better exploration.

\paragraph{Algorithm} Equipped with the sub-problem solver for policy improvement and the $Q$-function estimator, we are prepared to present the provable multi-agent PPO algorithm. The pseudo-code is listed in Algorithm~\ref{alg:1}. The algorithm runs for $K$ iterations. At the $k$-th iteration, we estimate $Q$-function for each agent $m\in \gN$ via the estimator (cf. Assumption~\ref{assump:estimator}) to measure the quality of $\bpi_\thetak$. The estimates would also serve as the ideal descent direction for policy improvement. Since we use a constrained parametric policy class, the ideal update is approximated with the best policy parameter $\theta \in \Theta$ by minimizing the MSE problem~\eqref{eq:MSE}, which runs SGD for $T$ iterations (cf. Algorithm~\ref{alg:3}). 
Thanks to the geometric characterization (cf. Lemma~\ref{lem:perf}), we are guaranteed to reach a globally improved total reward by updating each agent consecutively.

\begin{algorithm}
\caption{Multi-Agent PPO}
\label{alg:1}
\begin{algorithmic}[1]
\REQUIRE Markov game $( \gN, \gS, \bgA, \gP, r, \gamma )$, penalty parameter $\beta$, stepsize $\eta$ for sub-problem, number of SGD iterations $T$, number of iterations $K$.
\ENSURE Uniformly sample $k$ from $0,1, \cdots K-1$, return $\Bar{\bpi} = \bpi_{\theta_k}$.
\STATE Initialize $\theta_0^m=0$ for every $ m \in \gN$.
\FOR{$k=0, 1, \dots, K-1$}
    \STATE Set parameter $\beta_k \xleftarrow{} \beta\sqrt{K}$
    \FOR{$m=1, \cdots, N$}
        \STATE Sample $\{s_t, \rva^{1:m-1}_t, a_t^m\}_{t=0}^{T-1}$ from $ \sigma_k = \nu_k \bpi_\thetak$.
        \STATE Obtain $\hat{Q}_{\bpi_\thetak}^{1:m}(\sa, a^m)$ for each sample
        .
        \STATE Feed samples into Algorithm~\ref{alg:3}, obtain $\theta_{k+1}^m$.
        % \STATE Update policy: $\pi_{\theta_{k+1}}^m(\cdot|\sa) \propto \exp \{ \theta_{k+1}^\top \phi(\sa, \cdot)\}$
    \ENDFOR
\ENDFOR
\end{algorithmic}
\end{algorithm}

\subsection{Theoretical Analysis} 

Our analysis relies on problem-dependent quantities. We denote weighted $L_p$-norm of function $f$ on
state-space $\gX$ as $\|f\|_{p, \rho} = \left( \sum_{x \in \gX} \rho(x) |f(x)|^p \right)^{\frac{1}{p}}$
\begin{definition}
\label{def:concen}
At the $k$-th iteration, for $m \in \gN$ we define the following problem-dependent quantity using Radon-Nikodym derivatives
\begin{align*}
     \phi_k^m =   \left\|   \frac{ d( \nu_*\bpi_*^{1:m})}{d (\nu_k \bpi_{\thetak}^{1:m}) } \right\|_{2, \sigma_k}
\end{align*}
% \begin{align*}
%      \phi_k = \sup_{m \in \gN}  \mathop{\E}\limits_{\Tilde{\sigma}_k} &\bigg[ \bigg|  \frac{\nu_*(s)}{\nu_k(s)} \bigg( \frac{\bpi_*(\rva^{1:m-1},a^m|s)}{\bpi_0(\rva^{1:m-1}, a^m|s)} - \\
%     &\frac{\bpi_*(\rva^{1:m-1}|s) \cdot \bpi_k(a^m|\sa)}{\bpi_0(\rva^{1:m-1}, a^m|s)} \bigg) \bigg|^2 \bigg]^{\frac{1}{2}}
% \end{align*}
% and
% \begin{align*}
%      \varphi_k^m &= \left[ \E_{\sigma_k} \left| 
%      \frac{d( \nu_* \bpi_*^{1:m-1}) }{ d( \nu_k \bpi_{\theta_k}^{1:m-1}) }  \right|^2 \right]^{1/2}.
% \end{align*}
% \begin{align*}
%      \psi_k &= \sup_{m \in \gN} \mathop{\E}\limits_{\sigma_k}\\
%      &\bigg[ \bigg| 
%     \frac{\nu_*(s)}{\nu_k(s)} \bigg(
%     \frac{\bpi_*(\rva^{1:m-1},a^m|s)}{\bpi_k(\rva^{1:m-1}, a^m|s) } -
%      \frac{\bpi_*(\rva^{1:m-1}|s) }{\bpi_k(\rva^{1:m-1}|s) } \bigg) \bigg|^2 \bigg]^{\frac{1}{2}}.
% \end{align*}
These conditions are the well-known concentrability coefficients~\citep{munos2003, farahmand2010error,chen2019information} for the factorized policy. Still, our conditions are structurally simpler and weaker because they are only density ratios between stationary state-action distributions, not requiring trajectories to roll out.
\end{definition}

% Suppose for any agent $m \in \gN$ and $(\sa) \in \gS \times \gA^{m-1}$, policy improvement and policy evaluation errors satisfy,
% \begin{gather*}
%      \E_{\sigma_k} \Big( (\theta_{k+1}^m- \theta_k^m)^\top \phi- \beta_k^{-1} \hat{Q}_{\bpi_\thetak}^{1:m}  \Big)^2  \le (\epsilon_{k}^m)^2,\\
%     \E_{\sigma_k} \left( \hat{Q}_{\bpi_\thetak}^{1:m} - Q_{\bpi_\thetak}^{1:m} \right)^2  \le (\xi_k^m)^2
% \end{gather*}
% where we omit $(\sa, a^m)$ for simplicity.

 Now we are prepared to present the main theorem that characterizes the global convergence rate.

\begin{theorem}
\label{thm_mappo}
Under Assumption~\ref{assump:estimator}, for the output policy $\bar{\bpi}$ attained by Algorithm~\ref{alg:1} in the fully cooperative Markov game, set $\eta = \frac{R}{G\sqrt{T}}$ and
$$
\beta =  \sqrt{\frac{N B^2/2}{N\log{|\gA|} + \sum_{m=1}^N \sum_{k=0}^{K-1} (\Delta_k^m + \delta_k^m)}  }.
$$
 % let $G = 2\left( R +\frac{1}{(1-\gamma) \beta \sqrt{K}} \right)$. 
After $K$ iterations, we have $J(\bpi_*) - J(\bar{\bpi})$ upper bounded by
\begin{align*}
\gO \left(
    \frac{B \sqrt{N}}{1-\gamma} 
\sqrt{ 
    \frac{
         N \log{|\gA|} + \sum_{m=1}^N \sum_{k=0}^{K-1} (\Delta_k^m + \delta_k^m) 
        }{K}
}  
\right)
\end{align*}
% $$
%  \frac{N \beta^2 \log{|\gA|} + 2N/(1-\gamma)^2 + \beta^2 \sum_{m=1}^N \sum_{k=0}^{K-1} (\Delta_k^m + \delta_k^m)}{(1-\gamma)\beta \sqrt{K}}
% $$
where $\Delta_k^m = \sqrt{2}(\phi_{k}^m + \phi^{m-1}_k) \cdot \left(\epsilon_k^m + \frac{\xi_k^m}{\beta_k} \right)$ and $ \delta_k^m = 2 \phi_k^{m-1}\epsilon_{k}^m  $. Here $\epsilon_{k}^m
$ is the statistical error of a PPO iteration: for agent $m \in \gN$, 
\begin{gather*}
     \E_{\sigma_k} \Big( (\theta_{k+1}^m- \theta_k^m)^\top \phi- \beta_k^{-1} \hat{Q}_{\bpi_\thetak}^{1:m}  \Big)^2  \le (\epsilon_{k}^m)^2
\end{gather*}
where we omit $(\sa, a^m)$ for simplicity.

Let $\epsilon_{approx}$ be the approximation capability of the log-linear policy class we adopt, then $
\epsilon_{k}^m = \epsilon_{approx} + \gO(T^{-\frac{1}{4}}).
$ 
\end{theorem}

Theorem~\ref{thm_mappo} explicitly characterizes the performance of the output $\bar{\bpi}$ in terms of the number of iterations and the iteration errors. When PPO updates are ideal, namely,
viewing $\delta_k^\prime, \Delta_k^m$ to be $0$ for any $m \in \gN$, and $ k < K$ , the rate simplifies to $\gO \left(  \frac{N B }{1-\gamma} \sqrt{\frac{ \log{|\gA|}}{K}}\right)$. The dependency on iteration $K$ is $\gO(K^{-\frac{1}{2}})$, matching the same rate as the sample-based single-agent NPG analysis~\citep{agarwal2020optimality, liu2019neural}.

% Observing $\beta$, we find that it is desirable to use smaller step sizes when greater errors are accumulated during policy estimation and policy improvement, which aligns with deep reinforcement learning practice.

The proof of Theorem~\ref{thm_mappo} further requires the following parts: mirror-descent update analysis used in~\citep{liu2019neural} and Lemma~\ref{lem:perf} that builds sequential dependency structure among the agents.
The full proof is deferred to Appendix~\ref{sec:pf-ma}.

\subsection{Compare with Independent Learning}
In MARL, independent learning refers to a class of algorithms that train multiple agents independently. In these methods, each agent has its own policy function that maps the agent's observations to its actions. The policies are optimized using policy gradient methods in a decentralized manner without explicit communication or coordination, and without explicitly modeling the behavior of the other agents. Independent learning methods are widely used in MARL due to its strong performance and efficiency. 

In this subsection, we provide detailed comparisons between our algorithm and previous results on independent learning (both experiments and theories). We also performed a simulation study to showcase the superiority of our sequential policy update structure over naive independent policy gradient updates.
 
\paragraph{Experiments} Some empirical attempts showed independent policy gradient learning could achieve surprisingly strong performance in MARL, such as MAPPO~\citep{yu2021surprising},
IPPO~\citep{de2020independent}, and ~\citep{papoudakis1benchmarking}.

Despite the empirical success, these methods have several drawbacks. IPPO and MAPPO assume homogeneity (agents share the same action space and policy parameters). Thus, parameter sharing is required. Even though the parameter sharing can be turned off, they still suffer from no monotonic improvement guarantees, though being called PPO-based algorithms. Recall that the main virtue of vanilla TRPO~\citep{trpo} is monotonicity. Also, these methods do not come with any convergence guarantees. The converging problem becomes more severe when parameter-sharing is switched off. A counterexample in \citep[Proposition 1]{kuba2022trust} shows parameter sharing could lead to an exponentially-worse sub-optimal outcome.

Thanks to the sequential agents' structure and novel multi-agent mirror-decent analyses, we present the first MARL algorithm that converges at a sub-linear rate. Note that our results neither rely on the homogeneity of agents nor the value function decomposition rule.

\paragraph{Theories} Several theoretical works have studied convergence guarantees of independent policy optimization algorithms to a Nash equilibrium (NE) policy in MARL mathematically~\citep{daskalakis2020independent,leonardos2022global,fox2022independent,ding2022independent}.
Specifically, ~\citet{daskalakis2020independent} studied competitive RL. And others studied convergence to the NE policy in Markov potential games (an extension of fully-cooperative games). However, we argue that a NE policy is not necessarily optimal in terms of the value function.

In contrast to their work, we present the first provable multi-agent policy optimization algorithm that finds a policy with a near globally optimal value function equipped with a sub-linear convergence rate.

\paragraph{Simulation} To further validate the theoretical and experimental benefits of our algorithm, we conducted a numerical simulation to showcase the superiority of our algorithm with sequential updates structure over naive independent policy gradient updates. We consider von Neumann’s ratio game, a simple stochastic game also used by~\citet{daskalakis2020independent}. Simulation results show that, unlike our algorithm, the independent learning method has significant difficulty escaping the stationary point. Moreover, our algorithm consistently outperforms independent learning in maximizing value function. See Section~\ref{sec:sim} for detailed settings and results.

 \section{Pessimistic MA-PPO with Linear Function Approximation}
 \label{sec:pess}
In this section, we study the off-policy setting, using samples from a data distribution $\mu$ to evaluate $Q_\bpi$. Experimentally, since function approximators often cause a positive bias in value function~\cite{thrun1993issues}, many deep off-policy actor-critic algorithms introduce pessimism to reduce such overestimation~\citep{fujimoto2018addressing,laskin2020reinforcement}. We also adopt pessimistic policy evaluation in this setting, aligning with experimental works. 

We focus on the setting where value functions and policies are linearly parameterized. Our results can extend to the general function approximation setting, presented in Appendix~\ref{sec:pf-general}.

\begin{definition}[Linear Function Approximation] 
\label{def:linear}
Let $\phi$ be a set of feature mappings built conditionally, the same definition as Section~\ref{sec:main_thm}. Define the action value function class as $\gF^m = \{ \phi^\top \omega: \omega \in \sR^d, \|\omega\|_2 \le L, \phi^\top \omega  \in [0, \nicefrac{1}{1-\gamma}] \} $. The policy class is still parameterized by log-linear: $\Pi^m = \{ \pi \propto \exp(\phi^\top \theta) : \theta \in \sR^d, \|\theta\|_2 \le R\}$ (cf. Section~\ref{sec:main_thm}).
% , such that
% $$
% \pi_\theta (\cdot|\sa) \propto \exp{(\phi(s, \rva^{1:m-1}, \cdot)^\top \theta)}.
% $$
\end{definition}

\begin{remark} Under the definition, for any $m \in \gN$ and policy $\bpi$, there must exist a parameter $\omega \in \sR^d$ that satisfies
\[
Q_{\bpi}^{1:m}(s, \rva^{1:m}) = \phi(s, \rva^{1:m})^\top \omega
\]
 
\end{remark}

In this section, we fix the initial state at a certain $s_0$. Thus the expected reward we aim to maximize is defined as
\begin{align*}
    J(\bpi) \triangleq V_{\bpi}(s_0).
\end{align*}

Note that, in single-agent offline RL, only one policy affects the action at a particular state so that we can gauge the quality of value function estimates using an offline dataset $\gD$ consisting of states, actions, rewards, and transitions. Intuitively, when the following $L_0$ approaches 0,  we can say $f$ is a nice approximator for the $Q$-function~\citep{xie2021bellman}.
\begin{align*}
    L_0 = \frac{1}{n} \sum_{(s, a, r, s^\prime) \sim \gD} \left( f(s,a) -r- \gamma f(s^\prime, \pi) \right)^2
\end{align*}
where  $f(s', \pi)$ is a shorthand for $\sum_{a'} f(s', a') \pi(a'|s')$ which will be frequently used in this section. 

However, in the multi-agent environment, the complex dependent structure precludes the application of such an offline dataset. Specifically, for the $m$-th agent and policy $\bpi$, estimating the multi-agent value function $Q_\bpi^{1:m}$  demands that all agents not in $\{1:m\}$ must follow $\bpi$ (cf. Definition~\ref{def::1}), which could not be guaranteed by an offline dataset. 

Therefore, online interactions are unavoidable in the multi-agent setting we study. Below we make clarifications for the sample-generating protocol.

 We will collect state-action samples from a fixed  \emph{data distribution} $\mu = \mu_s \mu_a \in \Delta(\gS \times \bgA)  $. In the benign case, a well-covered $\mu$ guarantees adequate exploration over the whole state and action spaces.
 % Recall the definition of weighted 2-norm: $\|f\|^2_{2, \nu} = \E_\nu f^2$ where $f$ is any function $\in \gS \times \bgA \xrightarrow{} \sR$ and $\nu$ is a distribution over $\gS \times \bgA$. We shall also use $\|f\|^2_{2, \gD} = \frac{1}{n} \sum_{\gD} f(s,\rva)^2$ to represent the empirical version of $\|f\|^2_{2, \mu}$.
Assume we have access to a standard RL oracle 
\begin{definition}[Sampling Oracle]
\label{def:oracle}
  The oracle can start from $s \sim \mu_s$, take any action $\rva \in \bgA$, and obtain the next state $s^\prime \sim \gP(\cdot|s, \rva)$,  and reward $r(s, \rva)$. 
  \end{definition}

Our query oracle aligns with the classic \textbf{online sampling oracle}  for MDP~\citep{kakade2002approximately, du2019good,agarwal2020optimality}. The difference is that we transit for one step, while the classic online model usually terminates at the end of each episode.  We also note that our oracle is weaker than the  \textbf{generative model} ~\citep{kearns2002near,kakade2003sample,sidford2018near,li2020breaking} which assumes that agent can transit to \textbf{any}
state, thus greatly weakening the need for explicit exploration. Whereas our oracle starts from a fixed $\mu_s$.\footnote{In MDPs, such oracle is called $\mu$-reset model~\citep{kakade2002approximately}.}

We take advantage of the sampler in the following steps to obtain action value functions that preserve a small error under the multi-agent Bellman operator (cf. Definition~\ref{def:linear}). For agent $m \in \gN$ and $\bpi$, 
(1)  obtain $s \sim \mu_s$; (2) obtain $\rva \sim \mu_a$ and $\rva^\prime \sim \bpi^{m+1:N}(\cdot|s)$; (3) take $(\rva^{1:m}, \rva^\prime)$ as the joint action to query the oracle where $\rva^{1:m}$ represents the $\{1:m\}$ subset of $\rva$. The oracle returns $(r, s')$, which are guaranteed to satisfy:
 \begin{align*}
      r \sim  \mathop{\E}\limits_{\Tilde{\rva} \sim \bpi^{m+1:N}} R(s, \rva^{1:m}, \Tilde{\rva}),
      \quad
    s^\prime \sim \mathop{\E}\limits_{\Tilde{\rva} \sim \bpi^{m+1:N}} \gP(\cdot|s, \rva^{1:m},\Tilde{\rva}).
 \end{align*}

 Repeat these steps for $n$ times. Together this gives dataset $\gD^m = \{(s_i, \rva^{1:m}_i, r_i, s_i^\prime)|i = 1,2,\cdots n\}$.
 % Also let $\sigma^m = \{(s_i, \rva^{1:m}_i)|i = 1,2,\cdots n\}$ to denote the state-action pairs. 
 Define
 \[
    L^{1:m}(f^\prime, f, \bpi) \coloneqq \frac{1}{n} \sum_{\gD^m} \left( f^\prime(s,\rva^{1:m}) -r- \gamma f(s^\prime, \bpi^{1:m}) \right)^2
 \]
 where $f \in \gF^m$ (cf. Definition~\ref{def:linear}) and the summation is taken over $n$ quadruples of $(s, \rva^{1:m}, r, s^\prime)$.
 
 We will need the following Bellman error to evaluate the quality of $f$.
 \begin{equation}
 \label{eq_bellman}
      \gE^{1:m}(f, \bpi) = L^{1:m}(f, f, \bpi) - \min_{f^\prime \in \gF^m} L^{1:m}(f^\prime, f, \bpi).
 \end{equation}

 Intuitively, we consider $f$ as a nice approximation of  $Q_\bpi^{1:m}(s, \rva^{1:m})$ when the quantity is small. This formulation also works for general function approximation. See Appendix~\ref{sec:pf-general} for details.

 We shall need a concentrability measure accounting for the distributional mismatch.
 % Recall the definition of weighted 2-norm: $\|f\|^2_{2, \nu} = \E_\nu f^2$ where $f$ is any function $\in \gS \times \bgA \xrightarrow{} \sR$ and $\nu$ is a distribution over $\gS \times \bgA$.
 \begin{definition}[Concentrability]
The following condition characterizes the distribution shift from the $d_{\bpi_*}$ to the sampling distribution. 
\begin{align*}
    \gC^{d_{\bpi_*}}_\mu = \sup_{m \in \gN, f \in \gF^m, \bpi \in \Pi^m} \frac{\left\| f - \gT_{\bpi}^{1:m} f \right \|_{2, d_{\bpi_*}} }{ \left\| f - \gT_{\bpi}^{1:m} f \right \|_{2, \gD^m} }.
\end{align*}
\end{definition}
Recall that $\|\cdot\|_{2, \rho}$ is the weighted $L_2$-norm. In the nominator, the sum is taken over $(s, \rva^{1:m}) \sim d_{\bpi_*}$. Whereas in the denominator, the sum is taken over $(s, \rva^{1:m})$ from $\gD^m$ as an empirical version of $\mu$. The notion serves a similar role as concentrability coefficients in the literature~\citep{munos2003,agarwal2020optimality}: it measures the distributional mismatch between the underlying optimal distribution and the distribution of samples we employ.

\paragraph{Policy Evaluation}
 At the $k$-th iteration, we have the current policy $\bpi_\thetak$. We perform pessimistic policy evaluation via regularization to reduce value bias in evaluating $Q^{1:m}_{\bpi_\thetak}$. 
\[
\omega_k^m \xleftarrow[]{} \argmin_{\omega} \left(f(s_0, \bpi^{1:m}_k) + \lambda \gE^{1:m}(f, \bpi_\thetak) \right).
\]
Here $\gE$ is the Bellman error defined in~\eqref{eq_bellman}. We obtain $f_k^m = \phi^\top \omega_k^m$ as the pessimistic estimate for $Q^{1:m}_{\bpi_\thetak}$. This update has a closed-form solution under linear function approximation (cf. Definition~\ref{def:linear}). Moreover, under linear function approximation, the minimization on the right-hand side can be solved computationally efficiently because of its quadratic dependency on $\omega$. See details in Appendix~\ref{sec:pf-pess}

% In this case, $f(s_0, \bpi^{1:m})$ is instantiated as $\phi(s_0, \bpi^{1:m})^\top \omega$, and 
% \begin{align*}
%     &\gE^{1:m} (f, \bpi) = \sum ( \phi(s, \rva^{1:m})^\top \omega - r - \gamma \phi(s^\prime, \bpi^{1:m})^\top \omega)^2\\
%     &\quad - \min_{\omega^\prime}\sum ( \phi(s, \rva^{1:m})^\top \omega^\prime - r - \gamma \phi(s^\prime, \bpi^{1:m})^\top \omega)^2
% \end{align*}
% \begin{align*}
%     &L^{1:m}(f^\prime, f, \bpi)=\\
%     &\quad \frac{1}{n} \sum_{\gD^m} \left( \phi(s, \rva^{1:m})^\top \omega^\prime - r - \gamma \phi(s^\prime, \bpi^{1:m})^\top \omega \right)^2,
% \end{align*}
% summation is taken over $(s, \rva^{1:m}, r, s^\prime)$ from $\gD^m$.

\paragraph{Policy Improvement} When both value functions and policies are linear parameterized (cf. Definition~\ref{def:linear}), the mirror descent policy update for any $(s,\rva^{1:m}) \in \gS \times \gA^m$
\begin{align}
\label{eq:pess-upd}
    \pi_{k+1}^m &(a^m|\sa) \propto \pi_{k}^m(a^m|\sa) \cdot \exp( \eta f_k^m(s, \rva^{1:m}))
\end{align}
could be further simplified to parameter updates in $\sR^d$
\begin{align*}
   \theta_{k+1}^m = \theta_k^m + \eta \omega_k^m.
\end{align*}
This observation makes policy improvements in this setting significantly more superficial than in Section~\ref{sec:main_thm}. For the $k$-th iteration and agent $m \in \gN$, we only need to add $\eta \omega_k^m$ to the policy parameter $\theta_k^m$ to improve policy.

\paragraph{Algorithm} With the pessimistic policy evaluation and intuitive policy improvement, our pessimistic variant of the multi-agent PPO algorithm is presented in Algorithm~\ref{alg:2}.

\begin{algorithm}
\caption{Pessimistic Multi-Agent PPO with Linear Function Approximation}
\label{alg:2}
\begin{algorithmic}[1]
\REQUIRE Regularization coefficient $\lambda$.
\ENSURE Uniformly sample $k$ from $0, 1 \cdots K-1$, return $\bar{\bpi} = \bpi_\thetak$.
\STATE Initialize $\theta_0^m=0$ for every $ m \in \gN$.
\FOR{$k=0, 1,\dots, K-1$}
    \FOR{$m=1,2, \cdots, N$}
        \STATE Pessimistic policy evaluation:\\
        $\omega_k^m \xleftarrow[]{} \mathop{\argmin}\limits_{\omega} \left(f(s_0, \bpi^{1:m}_\thetak) + \lambda \gE^{1:m}(f, \bpi_\thetak)  \right)$.
        \STATE Policy improvement: $ \theta_{k+1}^m = \theta_k^m + \eta \omega_k^m$.
    \ENDFOR
\ENDFOR
\end{algorithmic}
\end{algorithm}

Now we are prepared to present the main theorem for this section.

\begin{theorem}
\label{thm:2}
For the output policy $\bar{\bpi}$ attained by Algorithm~\ref{alg:2} in a fully cooperative Markov game, set $\eta = (1-\gamma)\sqrt{\frac{\log{|\gA|}}{2K}}$ and $\lambda = (1-\gamma)^{-1}\left( \frac{d \log \frac{nLR}{\delta}}{n}\right)^{-\nicefrac{2}{3}}$. After $K$ iterations, w.p. at least $1-\delta$ we have $J(\bpi_*) - J(\bar{\bpi})$ upper bounded  by
\begin{align*}
     \gO \left( \frac{N}{(1-\gamma)^2} \sqrt{\frac{\log{|\gA|}}{K}}
    + \frac{\gC^{d_{\bpi_*}}_\mu}{(1-\gamma)^2} \sqrt[3]{\frac{d \log{\frac{nLR}{\delta}}}{n} }\right)
\end{align*}
\end{theorem}

To interpret this bound, the first term accounts for the optimization error accumulating from mirror descent updates~\eqref{eq:pess-upd}. The first term has an $(1-\gamma)^{-2}$
dependency on the discount factor, which may not be
tight, and we leave it as a future work to improve. 
The second term represents the estimation errors accumulated during training. We use state-action pairs from $\mu$ and the sampling oracle for minimizing $\gE^{1:m}(f, \bpi)$, thereby introducing \emph{distribution mismatch} which is expressed by $\gC^{d_{\bpi_*}}_\mu$. Note that this single-policy concentrability is already 
weaker than traditional concentrability coefficients~\citep{munos2003, farahmand2010error,perolat2015approximate}. Intuitively, a small
value of concentrability requires the data distribution $\mu$ close to $d_{\bpi_*}$, which is the unknown occupancy distribution of optimal policy. On the other hand, if $\gC^{d_{\bpi_*}}_\mu$ is large, then the bound becomes loose. We provide a similar result for general function approximation in the appendix (cf. Theorem~\ref{thm:3}).

 There is no explicit dependence on state-space $\gS$ in the theorem. Hence the online algorithm proves nice guarantees for function approximation even in the infinite-state setting.
 
To prove Theorem~\ref{thm:2}, the quantitative analysis for Bellman-consistent pessimism~\citep{xie2021bellman} is useful. We obtain statistical and convergence guarantees by taking advantage of the conditional dependency structure of the cooperative Markov games. See Appendix~\ref{sec:pf-pess} for details.

\section{Conclusion}
\label{sec:con}
In this work, we present a new multi-agent PPO algorithm that converges to the globally optimal policy at a sublinear rate. The key to the algorithm is a multi-agent performance difference lemma which enables sequential local policy updates. As a generalization, we extend the algorithm to the off-policy setting and present similar convergence guarantees. 
To our knowledge, this is the first 
 multi-agent PPO algorithm in cooperative Markov games that enjoys provable guarantees.

\section*{Acknowledgements}
JDL acknowledges support of the ARO under MURI Award W911NF-11-1-0304,  the Sloan Research Fellowship, NSF CCF 2002272, NSF IIS 2107304,  NSF CIF 2212262, ONR Young Investigator Award, and NSF CAREER Award 2144994.

\bibliographystyle{abbrvnat}
\bibliography{refs}

\begin{thebibliography}{67}
\providecommand{\natexlab}[1]{#1}
\providecommand{\url}[1]{\texttt{#1}}
\expandafter\ifx\csname urlstyle\endcsname\relax
  \providecommand{\doi}[1]{doi: #1}\else
  \providecommand{\doi}{doi: \begingroup \urlstyle{rm}\Url}\fi

\bibitem[Agarwal et~al.(2020)Agarwal, Kakade, Lee, and
  Mahajan]{agarwal2020optimality}
A.~Agarwal, S.~M. Kakade, J.~D. Lee, and G.~Mahajan.
\newblock Optimality and approximation with policy gradient methods in {M}arkov
  decision processes.
\newblock In \emph{Conference on Learning Theory}, pages 64--66. PMLR, 2020.

\bibitem[Antos et~al.(2008)Antos, Szepesv{\'a}ri, and Munos]{antos2008learning}
A.~Antos, C.~Szepesv{\'a}ri, and R.~Munos.
\newblock Learning near-optimal policies with {B}ellman-residual minimization
  based fitted policy iteration and a single sample path.
\newblock \emph{Machine Learning}, 71\penalty0 (1):\penalty0 89--129, 2008.

\bibitem[Branavan et~al.(2009)Branavan, Chen, Zettlemoyer, and
  Barzilay]{branavan2009reinforcement}
S.~R.~K. Branavan, H.~Chen, L.~S. Zettlemoyer, and R.~Barzilay.
\newblock Reinforcement learning for mapping instructions to actions.
\newblock In \emph{Proceedings of the Joint Conference of the 47th Annual
  Meeting of the ACL and the 4th International Joint Conference on Natural
  Language Processing of the AFNLP: Volume 1 - Volume 1}, ACL '09, page
  82–90, USA, 2009. Association for Computational Linguistics.

\bibitem[Brown and Sandholm(2018)]{brown2018superhuman}
N.~Brown and T.~Sandholm.
\newblock Superhuman {AI} for heads-up no-limit poker: Libratus beats top
  professionals.
\newblock \emph{Science}, 359\penalty0 (6374):\penalty0 418--424, 2018.

\bibitem[Cen et~al.(2021)Cen, Wei, and Chi]{cen2021fast}
S.~Cen, Y.~Wei, and Y.~Chi.
\newblock Fast policy extragradient methods for competitive games with entropy
  regularization.
\newblock In \emph{Advances in Neural Information Processing Systems}, pages
  27952--27964. Curran Associates, Inc., 2021.

\bibitem[Cen et~al.(2022)Cen, Chi, Du, and Xiao]{cen2022faster}
S.~Cen, Y.~Chi, S.~S. Du, and L.~Xiao.
\newblock Faster last-iterate convergence of policy optimization in zero-sum
  {M}arkov games.
\newblock \emph{arXiv preprint arXiv:2210.01050}, 2022.

\bibitem[Chen and Jiang(2019)]{chen2019information}
J.~Chen and N.~Jiang.
\newblock Information-theoretic considerations in batch reinforcement learning.
\newblock In \emph{International Conference on Machine Learning}, pages
  1042--1051. PMLR, 2019.

\bibitem[Daskalakis et~al.(2020)Daskalakis, Foster, and
  Golowich]{daskalakis2020independent}
C.~Daskalakis, D.~J. Foster, and N.~Golowich.
\newblock Independent policy gradient methods for competitive reinforcement
  learning.
\newblock \emph{Advances in neural information processing systems},
  33:\penalty0 5527--5540, 2020.

\bibitem[de~Witt et~al.(2020)de~Witt, Gupta, Makoviichuk, Makoviychuk, Torr,
  Sun, and Whiteson]{de2020independent}
C.~S. de~Witt, T.~Gupta, D.~Makoviichuk, V.~Makoviychuk, P.~H. Torr, M.~Sun,
  and S.~Whiteson.
\newblock Is independent learning all you need in the {S}tar{C}raft multi-agent
  challenge?
\newblock \emph{arXiv preprint arXiv:2011.09533}, 2020.

\bibitem[Ding et~al.(2022)Ding, Wei, Zhang, and Jovanovic]{ding2022independent}
D.~Ding, C.-Y. Wei, K.~Zhang, and M.~Jovanovic.
\newblock Independent policy gradient for large-scale {M}arkov potential games:
  Sharper rates, function approximation, and game-agnostic convergence.
\newblock In \emph{International Conference on Machine Learning}, pages
  5166--5220. PMLR, 2022.

\bibitem[Du et~al.(2019)Du, Kakade, Wang, and Yang]{du2019good}
S.~S. Du, S.~M. Kakade, R.~Wang, and L.~F. Yang.
\newblock Is a good representation sufficient for sample efficient
  reinforcement learning?
\newblock \emph{arXiv preprint arXiv:1910.03016}, 2019.

\bibitem[Duan et~al.(2016)Duan, Chen, Houthooft, Schulman, and
  Abbeel]{duan2016benchmarking}
Y.~Duan, X.~Chen, R.~Houthooft, J.~Schulman, and P.~Abbeel.
\newblock Benchmarking deep reinforcement learning for continuous control.
\newblock In \emph{International conference on machine learning}, pages
  1329--1338. PMLR, 2016.

\bibitem[Fan et~al.(2020)Fan, Wang, Xie, and Yang]{fan2020theoretical}
J.~Fan, Z.~Wang, Y.~Xie, and Z.~Yang.
\newblock A theoretical analysis of deep {Q}-learning.
\newblock In \emph{Learning for Dynamics and Control}, pages 486--489. PMLR,
  2020.

\bibitem[Farahmand et~al.(2010)Farahmand, Szepesv{\'a}ri, and
  Munos]{farahmand2010error}
A.-m. Farahmand, C.~Szepesv{\'a}ri, and R.~Munos.
\newblock Error propagation for approximate policy and value iteration.
\newblock \emph{Advances in Neural Information Processing Systems}, 23, 2010.

\bibitem[Foerster et~al.(2018)Foerster, Farquhar, Afouras, Nardelli, and
  Whiteson]{foerster2018counterfactual}
J.~Foerster, G.~Farquhar, T.~Afouras, N.~Nardelli, and S.~Whiteson.
\newblock Counterfactual multi-agent policy gradients.
\newblock In \emph{Proceedings of the AAAI conference on artificial
  intelligence}, volume~32, 2018.

\bibitem[Fox et~al.(2022)Fox, Mcaleer, Overman, and
  Panageas]{fox2022independent}
R.~Fox, S.~M. Mcaleer, W.~Overman, and I.~Panageas.
\newblock Independent natural policy gradient always converges in markov
  potential games.
\newblock In \emph{International Conference on Artificial Intelligence and
  Statistics}, pages 4414--4425. PMLR, 2022.

\bibitem[Fujimoto et~al.(2018)Fujimoto, Hoof, and
  Meger]{fujimoto2018addressing}
S.~Fujimoto, H.~Hoof, and D.~Meger.
\newblock Addressing function approximation error in actor-critic methods.
\newblock In \emph{International conference on machine learning}, pages
  1587--1596. PMLR, 2018.

\bibitem[Gimpel and Smith(2010)]{gimpel2010softmax}
K.~Gimpel and N.~A. Smith.
\newblock Softmax-margin {CRF}s: {T}raining log-linear models with cost
  functions.
\newblock In \emph{Human Language Technologies: The 2010 Annual Conference of
  the North American Chapter of the Association for Computational Linguistics},
  pages 733--736, 2010.

\bibitem[Guo et~al.(2016)Guo, Singh, Lewis, and Lee]{guo2016deep}
X.~Guo, S.~Singh, R.~Lewis, and H.~Lee.
\newblock Deep learning for reward design to improve {M}onte {C}arlo tree
  search in {A}tari games.
\newblock \emph{arXiv preprint arXiv:1604.07095}, 2016.

\bibitem[Heess et~al.(2013)Heess, Silver, and Teh]{heess2013actor}
N.~Heess, D.~Silver, and Y.~W. Teh.
\newblock Actor-critic reinforcement learning with energy-based policies.
\newblock In \emph{European Workshop on Reinforcement Learning}, pages 45--58.
  PMLR, 2013.

\bibitem[Jin et~al.(2017)Jin, Ge, Netrapalli, Kakade, and
  Jordan]{jin2017escape}
C.~Jin, R.~Ge, P.~Netrapalli, S.~M. Kakade, and M.~I. Jordan.
\newblock How to escape saddle points efficiently.
\newblock In \emph{International conference on machine learning}, pages
  1724--1732. PMLR, 2017.

\bibitem[Jin et~al.(2021)Jin, Yang, and Wang]{jin2021pessimism}
Y.~Jin, Z.~Yang, and Z.~Wang.
\newblock Is pessimism provably efficient for offline {RL}?
\newblock In \emph{International Conference on Machine Learning}, pages
  5084--5096. PMLR, 2021.

\bibitem[Kakade and Langford(2002)]{kakade2002approximately}
S.~Kakade and J.~Langford.
\newblock Approximately optimal approximate reinforcement learning.
\newblock In \emph{In Proc. 19th International Conference on Machine Learning}.
  Citeseer, 2002.

\bibitem[Kakade(2001)]{kakade2001natural}
S.~M. Kakade.
\newblock A natural policy gradient.
\newblock \emph{Advances in neural information processing systems}, 14, 2001.

\bibitem[Kakade(2003)]{kakade2003sample}
S.~M. Kakade.
\newblock \emph{On the sample complexity of reinforcement learning}.
\newblock University of London, University College London (United Kingdom),
  2003.

\bibitem[Kearns and Singh(2002)]{kearns2002near}
M.~Kearns and S.~Singh.
\newblock Near-optimal reinforcement learning in polynomial time.
\newblock \emph{Machine learning}, 49\penalty0 (2):\penalty0 209--232, 2002.

\bibitem[Kraemer and Banerjee(2016)]{kraemer2016multi}
L.~Kraemer and B.~Banerjee.
\newblock Multi-agent reinforcement learning as a rehearsal for decentralized
  planning.
\newblock \emph{Neurocomputing}, 190:\penalty0 82--94, 2016.

\bibitem[Kuba et~al.(2022)Kuba, Chen, Wen, Wen, Sun, Wang, and
  Yang]{kuba2022trust}
J.~G. Kuba, R.~Chen, M.~Wen, Y.~Wen, F.~Sun, J.~Wang, and Y.~Yang.
\newblock Trust region policy optimisation in multi-agent reinforcement
  learning.
\newblock In \emph{International Conference on Learning Representations}, 2022.
\newblock URL \url{https://openreview.net/forum?id=EcGGFkNTxdJ}.

\bibitem[Laskin et~al.(2020)Laskin, Lee, Stooke, Pinto, Abbeel, and
  Srinivas]{laskin2020reinforcement}
M.~Laskin, K.~Lee, A.~Stooke, L.~Pinto, P.~Abbeel, and A.~Srinivas.
\newblock Reinforcement learning with augmented data.
\newblock \emph{Advances in neural information processing systems},
  33:\penalty0 19884--19895, 2020.

\bibitem[Lee et~al.(2020)Lee, Fischer, Liu, Guo, Lee, Canny, and
  Guadarrama]{lee2020predictive}
K.-H. Lee, I.~Fischer, A.~Liu, Y.~Guo, H.~Lee, J.~Canny, and S.~Guadarrama.
\newblock Predictive information accelerates learning in {RL}.
\newblock \emph{Advances in Neural Information Processing Systems},
  33:\penalty0 11890--11901, 2020.

\bibitem[Leonardos et~al.(2022)Leonardos, Overman, Panageas, and
  Piliouras]{leonardos2022global}
S.~Leonardos, W.~Overman, I.~Panageas, and G.~Piliouras.
\newblock Global convergence of multi-agent policy gradient in markov potential
  games.
\newblock In \emph{International Conference on Learning Representations}, 2022.
\newblock URL \url{https://openreview.net/forum?id=gfwON7rAm4}.

\bibitem[Li et~al.(2020)Li, Wei, Chi, Gu, and Chen]{li2020breaking}
G.~Li, Y.~Wei, Y.~Chi, Y.~Gu, and Y.~Chen.
\newblock Breaking the sample size barrier in model-based reinforcement
  learning with a generative model.
\newblock \emph{Advances in neural information processing systems},
  33:\penalty0 12861--12872, 2020.

\bibitem[Lillicrap et~al.(2015)Lillicrap, Hunt, Pritzel, Heess, Erez, Tassa,
  Silver, and Wierstra]{lillicrap2015continuous}
T.~P. Lillicrap, J.~J. Hunt, A.~Pritzel, N.~Heess, T.~Erez, Y.~Tassa,
  D.~Silver, and D.~Wierstra.
\newblock Continuous control with deep reinforcement learning.
\newblock \emph{arXiv preprint arXiv:1509.02971}, 2015.

\bibitem[Littman(1994)]{littman1994markov}
M.~L. Littman.
\newblock Markov games as a framework for multi-agent reinforcement learning.
\newblock In \emph{Machine learning proceedings 1994}, pages 157--163.
  Elsevier, 1994.

\bibitem[Liu et~al.(2019)Liu, Cai, Yang, and Wang]{liu2019neural}
B.~Liu, Q.~Cai, Z.~Yang, and Z.~Wang.
\newblock Neural trust region/proximal policy optimization attains globally
  optimal policy.
\newblock \emph{Advances in neural information processing systems}, 32, 2019.

\bibitem[Liu et~al.(2020)Liu, Swaminathan, Agarwal, and
  Brunskill]{liu2020provably}
Y.~Liu, A.~Swaminathan, A.~Agarwal, and E.~Brunskill.
\newblock Provably good batch off-policy reinforcement learning without great
  exploration.
\newblock \emph{Advances in neural information processing systems},
  33:\penalty0 1264--1274, 2020.

\bibitem[Lowe et~al.(2017)Lowe, Wu, Tamar, Harb, Pieter~Abbeel, and
  Mordatch]{lowe2017multi}
R.~Lowe, Y.~I. Wu, A.~Tamar, J.~Harb, O.~Pieter~Abbeel, and I.~Mordatch.
\newblock Multi-agent actor-critic for mixed cooperative-competitive
  environments.
\newblock \emph{Advances in neural information processing systems}, 30, 2017.

\bibitem[Mnih et~al.(2015)Mnih, Kavukcuoglu, Silver, Rusu, Veness, Bellemare,
  Graves, Riedmiller, Fidjeland, Ostrovski, et~al.]{mnih2015human}
V.~Mnih, K.~Kavukcuoglu, D.~Silver, A.~A. Rusu, J.~Veness, M.~G. Bellemare,
  A.~Graves, M.~Riedmiller, A.~K. Fidjeland, G.~Ostrovski, et~al.
\newblock Human-level control through deep reinforcement learning.
\newblock \emph{nature}, 518\penalty0 (7540):\penalty0 529--533, 2015.

\bibitem[Moskovitz et~al.(2021)Moskovitz, Parker-Holder, Pacchiano, Arbel, and
  Jordan]{moskovitz2021tactical}
T.~Moskovitz, J.~Parker-Holder, A.~Pacchiano, M.~Arbel, and M.~Jordan.
\newblock Tactical optimism and pessimism for deep reinforcement learning.
\newblock \emph{Advances in Neural Information Processing Systems},
  34:\penalty0 12849--12863, 2021.

\bibitem[Munos(2003)]{munos2003}
R.~Munos.
\newblock Error bounds for approximate policy iteration.
\newblock In \emph{International Conference on Machine Learning}, page
  560–567, 2003.

\bibitem[Munos and Szepesv{\'a}ri(2008)]{munos2008finite}
R.~Munos and C.~Szepesv{\'a}ri.
\newblock Finite-time bounds for fitted value iteration.
\newblock \emph{Journal of Machine Learning Research}, 9\penalty0 (5), 2008.

\bibitem[Nesterov(2003)]{nesterov2003introductory}
Y.~Nesterov.
\newblock \emph{Introductory lectures on convex optimization: {A} basic
  course}, volume~87.
\newblock Springer Science \& Business Media, 2003.

\bibitem[Papoudakis et~al.(2021)Papoudakis, Christianos, Sch{\"a}fer, and
  Albrecht]{papoudakis1benchmarking}
G.~Papoudakis, F.~Christianos, L.~Sch{\"a}fer, and S.~V. Albrecht.
\newblock Benchmarking multi-agent deep reinforcement learning algorithms in
  cooperative tasks.
\newblock In \emph{Thirty-fifth Conference on Neural Information Processing
  Systems Datasets and Benchmarks Track (Round 1)}, 2021.

\bibitem[Perolat et~al.(2015)Perolat, Scherrer, Piot, and
  Pietquin]{perolat2015approximate}
J.~Perolat, B.~Scherrer, B.~Piot, and O.~Pietquin.
\newblock Approximate dynamic programming for two-player zero-sum {M}arkov
  games.
\newblock In \emph{International Conference on Machine Learning}, pages
  1321--1329, 2015.

\bibitem[Rashidinejad et~al.(2021)Rashidinejad, Zhu, Ma, Jiao, and
  Russell]{rashidinejad2021bridging}
P.~Rashidinejad, B.~Zhu, C.~Ma, J.~Jiao, and S.~Russell.
\newblock Bridging offline reinforcement learning and imitation learning: {A}
  tale of pessimism.
\newblock \emph{Advances in Neural Information Processing Systems},
  34:\penalty0 11702--11716, 2021.

\bibitem[Schulman et~al.(2015)Schulman, Levine, Abbeel, Jordan, and
  Moritz]{trpo}
J.~Schulman, S.~Levine, P.~Abbeel, M.~Jordan, and P.~Moritz.
\newblock Trust region policy optimization.
\newblock In \emph{International conference on machine learning}, pages
  1889--1897. PMLR, 2015.

\bibitem[Schulman et~al.(2017)Schulman, Wolski, Dhariwal, Radford, and
  Klimov]{ppo}
J.~Schulman, F.~Wolski, P.~Dhariwal, A.~Radford, and O.~Klimov.
\newblock Proximal policy optimization algorithms.
\newblock \emph{arXiv preprint arXiv:1707.06347}, 2017.

\bibitem[Shalev-Shwartz and Ben-David(2014)]{shalev2014understanding}
S.~Shalev-Shwartz and S.~Ben-David.
\newblock \emph{Understanding machine learning: {F}rom theory to algorithms}.
\newblock Cambridge university press, 2014.

\bibitem[Shani et~al.(2020)Shani, Efroni, and Mannor]{shani2020adaptive}
L.~Shani, Y.~Efroni, and S.~Mannor.
\newblock Adaptive trust region policy optimization: Global convergence and
  faster rates for regularized {MDP}s.
\newblock In \emph{Proceedings of the AAAI Conference on Artificial
  Intelligence}, pages 5668--5675, 2020.

\bibitem[Shapley(1953)]{shapley1953stochastic}
L.~S. Shapley.
\newblock Stochastic games.
\newblock \emph{Proceedings of the national academy of sciences}, 39\penalty0
  (10):\penalty0 1095--1100, 1953.

\bibitem[Sidford et~al.(2018)Sidford, Wang, Wu, Yang, and Ye]{sidford2018near}
A.~Sidford, M.~Wang, X.~Wu, L.~F. Yang, and Y.~Ye.
\newblock Near-optimal time and sample complexities for solving discounted
  markov decision process with a generative model.
\newblock \emph{arXiv preprint arXiv:1806.01492}, 2018.

\bibitem[Silver et~al.(2016)Silver, Huang, Maddison, Guez, Sifre, Van
  Den~Driessche, Schrittwieser, Antonoglou, Panneershelvam, Lanctot,
  et~al.]{silver2016mastering}
D.~Silver, A.~Huang, C.~J. Maddison, A.~Guez, L.~Sifre, G.~Van Den~Driessche,
  J.~Schrittwieser, I.~Antonoglou, V.~Panneershelvam, M.~Lanctot, et~al.
\newblock Mastering the game of {G}o with deep neural networks and tree search.
\newblock \emph{nature}, 529\penalty0 (7587):\penalty0 484--489, 2016.

\bibitem[Silver et~al.(2017)Silver, Schrittwieser, Simonyan, Antonoglou, Huang,
  Guez, Hubert, Baker, Lai, Bolton, et~al.]{silver2017mastering}
D.~Silver, J.~Schrittwieser, K.~Simonyan, I.~Antonoglou, A.~Huang, A.~Guez,
  T.~Hubert, L.~Baker, M.~Lai, A.~Bolton, et~al.
\newblock Mastering the game of {G}o without human knowledge.
\newblock \emph{nature}, 550\penalty0 (7676):\penalty0 354--359, 2017.

\bibitem[Sutton et~al.(1999)Sutton, McAllester, Singh, and
  Mansour]{sutton1999policy}
R.~S. Sutton, D.~McAllester, S.~Singh, and Y.~Mansour.
\newblock Policy gradient methods for reinforcement learning with function
  approximation.
\newblock \emph{Advances in neural information processing systems}, 12, 1999.

\bibitem[Thrun and Schwartz(1993)]{thrun1993issues}
S.~Thrun and A.~Schwartz.
\newblock Issues in using function approximation for reinforcement learning.
\newblock In \emph{Proceedings of the Fourth Connectionist Models Summer
  School}, volume 255, page 263. Hillsdale, NJ, 1993.

\bibitem[Tian et~al.(2019)Tian, Ma, Gong, Sengupta, Chen, Pinkerton, and
  Zitnick]{tian2019elf}
Y.~Tian, J.~Ma, Q.~Gong, S.~Sengupta, Z.~Chen, J.~Pinkerton, and L.~Zitnick.
\newblock Elf opengo: An analysis and open reimplementation of {A}lpha{Z}ero.
\newblock In \emph{International Conference on Machine Learning}, pages
  6244--6253. PMLR, 2019.

\bibitem[Uehara and Sun(2021)]{uehara2021pessimistic}
M.~Uehara and W.~Sun.
\newblock Pessimistic model-based offline {RL}: {P}ac bounds and posterior
  sampling under partial coverage.
\newblock \emph{arXiv e-prints}, pages arXiv--2107, 2021.

\bibitem[Vinyals et~al.(2019)Vinyals, Babuschkin, Czarnecki, Mathieu, Dudzik,
  Chung, Choi, Powell, Ewalds, Georgiev, et~al.]{vinyals2019grandmaster}
O.~Vinyals, I.~Babuschkin, W.~M. Czarnecki, M.~Mathieu, A.~Dudzik, J.~Chung,
  D.~H. Choi, R.~Powell, T.~Ewalds, P.~Georgiev, et~al.
\newblock Grandmaster level in {S}tar{C}raft ii using multi-agent reinforcement
  learning.
\newblock \emph{Nature}, 575\penalty0 (7782):\penalty0 350--354, 2019.

\bibitem[Wen et~al.(2019)Wen, Yang, Luo, Wang, and Pan]{wen2019probabilistic}
Y.~Wen, Y.~Yang, R.~Luo, J.~Wang, and W.~Pan.
\newblock Probabilistic recursive reasoning for multi-agent reinforcement
  learning.
\newblock \emph{arXiv preprint arXiv:1901.09207}, 2019.

\bibitem[Xie and Jiang(2020)]{xie2020q}
T.~Xie and N.~Jiang.
\newblock Q* approximation schemes for batch reinforcement learning: {A}
  theoretical comparison.
\newblock In \emph{Conference on Uncertainty in Artificial Intelligence}, pages
  550--559. PMLR, 2020.

\bibitem[Xie et~al.(2021)Xie, Cheng, Jiang, Mineiro, and
  Agarwal]{xie2021bellman}
T.~Xie, C.-A. Cheng, N.~Jiang, P.~Mineiro, and A.~Agarwal.
\newblock Bellman-consistent pessimism for offline reinforcement learning.
\newblock \emph{Advances in neural information processing systems},
  34:\penalty0 6683--6694, 2021.

\bibitem[Yang et~al.(2018)Yang, Luo, Li, Zhou, Zhang, and Wang]{yang2018mean}
Y.~Yang, R.~Luo, M.~Li, M.~Zhou, W.~Zhang, and J.~Wang.
\newblock Mean field multi-agent reinforcement learning.
\newblock In \emph{International conference on machine learning}, pages
  5571--5580. PMLR, 2018.

\bibitem[Yu et~al.(2021)Yu, Velu, Vinitsky, Wang, Bayen, and
  Wu]{yu2021surprising}
C.~Yu, A.~Velu, E.~Vinitsky, Y.~Wang, A.~Bayen, and Y.~Wu.
\newblock The surprising effectiveness of {PPO} in cooperative, multi-agent
  games.
\newblock \emph{arXiv preprint arXiv:2103.01955}, 2021.

\bibitem[Zhan et~al.(2022)Zhan, Huang, Huang, Jiang, and Lee]{zhan2022offline}
W.~Zhan, B.~Huang, A.~Huang, N.~Jiang, and J.~Lee.
\newblock Offline reinforcement learning with realizability and single-policy
  concentrability.
\newblock In \emph{Conference on Learning Theory}, pages 2730--2775. PMLR,
  2022.

\bibitem[Zhang et~al.(2020)Zhang, Chen, Huang, Li, Yang, Zhang, and
  Wang]{zhang2020bi}
H.~Zhang, W.~Chen, Z.~Huang, M.~Li, Y.~Yang, W.~Zhang, and J.~Wang.
\newblock Bi-level actor-critic for multi-agent coordination.
\newblock In \emph{Proceedings of the AAAI Conference on Artificial
  Intelligence}, pages 7325--7332, 2020.

\bibitem[Zhang et~al.(2021)Zhang, Yang, and Ba{\c{s}}ar]{zhang2021multi}
K.~Zhang, Z.~Yang, and T.~Ba{\c{s}}ar.
\newblock Multi-agent reinforcement learning: {A} selective overview of
  theories and algorithms.
\newblock \emph{Handbook of Reinforcement Learning and Control}, pages
  321--384, 2021.

\bibitem[Zhao et~al.(2022)Zhao, Tian, Lee, and Du]{zhao2022provably}
Y.~Zhao, Y.~Tian, J.~Lee, and S.~Du.
\newblock Provably efficient policy optimization for two-player zero-sum
  {M}arkov games.
\newblock In \emph{International Conference on Artificial Intelligence and
  Statistics}, pages 2736--2761. PMLR, 2022.

\end{thebibliography}

\newpage
\appendix

\section{Sub-problem Solver for Section~\ref{sec:main_thm}}
\label{sec:pf-alg}
\begin{algorithm}
\caption{Policy Improvement Solver for MA-PPO}
\label{alg:3}
\begin{algorithmic}[1]
\REQUIRE MG $( \gN, \gS, \bgA, \gP, r, \gamma )$, iterations $T$, stepsize $\eta$, samples $\{s_t, \rva^{1:m-1}_t, a_t^m \}_{t=0}^{T-1}$.
\ENSURE Policy update $\theta$.
% \STATE Set $\eta \xleftarrow{} \nicefrac{1}{\sqrt{T}}$.
\STATE Initialize $\theta_0=0$.
\FOR{$t=0, 1, \dots, T-1$}
        \STATE Let $(s, \rva^{1:m-1}, a) \xleftarrow{} (s_t, \rva^{1:m-1}_t, a_t^m)$.
        \STATE 
        $
            \theta(t+\frac{1}{2}) \xleftarrow{} \theta(t)- 2\eta \phi(s, \rva^{1:m-1}, a) \left( \left(\theta(t)-\theta_k^m \right)^\top \phi(\sa, a^m) -  \beta_k^{-1} \hat{Q}_{\bpi_k}^{1:m}(\sa, a^m) \right).
        $
        \STATE $\theta(t+1) \xleftarrow{} \Pi_{\Theta} \theta(t + \frac{1}{2})$
\ENDFOR
\STATE Calculate average: $\bar{\theta} \xleftarrow{} \frac{1}{T} \sum_{t=1}^{T} \theta_t.$
\end{algorithmic}
\end{algorithm}

\section{Proofs for Section~\ref{sec:main_thm}}
\label{sec:pf-ma}
First, we note that using SGD updates to solve the MSE problem has the following guarantee.
\begin{lemma}[Average policy]
\label{lem:aver-policy}
For a convex objective function $F(\theta)$, suppose the gradient is bounded by G, and the output $\bar{\theta}$ converges to the best function in the class at
$$
 F(\bar{\theta}) - \min_{\|\theta\| \le R} F(\theta) \le \frac{GR}{\sqrt{T}}
$$
where we set $\eta = \frac{R}{G\sqrt{T}}$.
\end{lemma}
\begin{proof}
Please refer to Theorem 14.8~\citep{shalev2014understanding}.
\end{proof}

Now we turn to Algorithm~\ref{alg:3}, in which, we feed samples $\{s_t, \rva^{1:m-1}_t, a_t^m\}_{t=0}^{T-1}$ from $ \sigma_k = \nu_k \bpi_{\theta_k}$ into the algorithm (for $m \in \gN$), in order to minimize
$$
    L(\theta^m) = \E_{\sigma_k} \left((\theta^m)^\top \phi(s, \rva^{1:m-1}, a^m) - (\beta_k^{-1} \hat{Q}_{\bpi_\thetak}^{1:m}(s, \rva^{1:m-1}, a^m) + (\theta_k^m)^\top \phi(s, \rva^{1:m-1}, a^m)) \right)^2.
$$

We have the following theoretical guarantee for the algorithm.
\begin{lemma}[Policy Improvement error] At the $k$-th outer loop, the output policy $\theta_{k+1}$ from Algorithm~\ref{alg:3} satisfies 
\begin{align*}
    \sqrt{ L(\theta_{k+1}^m) } \le \epsilon_{k}^m
\end{align*}
where $\epsilon_{k}^m = \epsilon_{approx} + \gO(T^{-\frac{1}{4}})$.

\end{lemma}
 
\begin{proof}
For Algorithm~\ref{alg:3}, we have $\|\phi\|_2 \le 1$ and $\Theta = \{\|\theta\| \le R | \theta \in \sR^d  \}$.Thereby the gradient of $L(\theta)$ is bounded by
$$
G = 2 \left(R + \frac{1}{(1-\gamma) \beta_k}\right).
$$
From Lemma~\ref{lem:aver-policy}, we have the following guarantee holds for any outer iteration $k < K$ 
$$
L(\theta_{k+1}^m) \le = \min_\theta L(\theta) + \gO(1/\sqrt{T}),
$$
when we set $\eta = \frac{R}{G\sqrt{T}}$. Thus
\begin{align*}
    \epsilon_{k}^m = \sqrt{\min_\theta L(\theta)} + \gO(T^{-\frac{1}{4}})= \epsilon_{approx} + \gO(T^{-\frac{1}{4}}).
\end{align*}
The proof is completed.
\end{proof}

\begin{lemma}[Multi-Agent Advantage Decomposition]
\label{lem:decom}
In cooperative Markov games, the following decomposition holds for any joint policy $\pi$, state $s$, and agents $1:m$,
\begin{equation*}
    A_{\bpi}^{1:m}(s, \rva^{1:m}) = \sum_{i=1}^m A_{\bpi}^{i}(s, \rva^{1:i-1}, a^{i})
\end{equation*}

\end{lemma}
\begin{proof}
    Please refer to Lemma 1~\citep{kuba2022trust}.
\end{proof}

\paragraph{Proof for Lemma~\ref{lem:perf}}
\begin{proof}
From the classical performance difference lemma~\citep[Lemma 6.1]{kakade2002approximately} we have
\begin{align*}
    J(\bpi_*) - J(\bpi) &= \frac{1}{1-\gamma} \mathop{\E}\limits_{\sigma_*} A_{\bpi}^{1:N}(s, \rva^{1:N})
\end{align*}
Decomposing the all-agents advantage function into individual contributions via the multi-agent advantage decomposition lemma (cf. Lemma~\ref{lem:decom}), we have

\begin{align*}
J(\bpi_*) - J(\bpi) &= \frac{1}{1-\gamma} \mathop{\E}\limits_{\sigma_*} A_{\bpi}^{1:N}(s, \rva^{1:N})\\ 
&= \frac{1}{1-\gamma} \sum_{m=1}^N \E_{s \sim \nu_*} \E_{\rva^{1:m-1} \sim \bpi_{*}^{1:m-1}}\left\langle A_{\bpi}^{m}(\sa, a^{m}), \pi_*^m(\cdot| \sa) \right \rangle.
\end{align*}
Note that we have $\sum_a \pi^m(a|s, a^{1:m-1}) A_{\bpi}^{m} (\sa,a) = 0 $, then
\begin{align*}
&\quad \frac{1}{1-\gamma} \sum_{m=1}^N \E_{s \sim \nu_*} \E_{\rva^{1:m-1} \sim \bpi_{*}^{1:m-1}}\left\langle A_{\bpi}^{m} (\sa, \cdot), \pi_*^m(\cdot|\sa) - \pi^m(\cdot|s, a^{1:m-1}) \right\rangle\\
&= \frac{1}{1-\gamma} \sum_{m=1}^N \E_{s \sim \nu_*} \E_{\rva^{1:m-1} \sim \bpi_{*}^{1:m-1}}\left\langle Q_{\bpi}^{1:m} (\sa, \cdot), \pi_*^m(\cdot|\sa) - \pi^m(\cdot|s, a^{1:m-1}) \right\rangle
\end{align*}
where the last line is because $A_{\bpi}^{m} (\sa, a^m) = Q_{\bpi}^{1:m} (\sa, a^m) - Q_{\bpi}^{1:m-1} (\sa)$ and $Q_{\bpi}^{1:m-1} (\sa)$ can be omitted because it does not change with $a^m$.

\end{proof}

\paragraph{Proof for Proposition~\ref{prop:update}.} 
\begin{proof} For any $(\sa) \in \gS \times \gA^{m-1}$, policy $\hat{\pi}^m_{k+1}(\cdot| s, \rva^{1:m-1})$ is obtained via
\begin{gather*}
    \max_{\pi^m} \E_{\nu_k} \left[\langle \hat{Q}_{\bpi_k}^{1:m}(\sa, \cdot), \pi^m (\cdot|s, \rva^{1:m-1})\rangle -\beta_k KL\left(\pi^m(\cdot|\sa)\| \pi_{\theta_k}(\cdot|\sa)\right) \right]\\
    \text{s.t.} \quad  \sum_{a^m \in \gA} \pi^m (a^m|s, \rva^{1:m-1}) = 1
\end{gather*}
 
Adding constraint as a Lagrangian multiplier, we have 
\begin{align*}
    \int_{\gS \times \gA^{m-1}} &\left[\langle \hat{Q}_{\bpi_k}^{1:m}(\sa, \cdot), \pi^m (\cdot|s, \rva^{1:m-1})\rangle -\beta_k KL\left(\pi^m(\cdot|\sa)\| \pi_{\theta_k^m}(\cdot|\sa)\right) \right] \sigma_k ds d\rva^{1:m-1}\\
    &+ \int_{\gS \times \gA^{m-1}} \left( \sum_{a^m \in \gA} \pi^m (a^m|s, \rva^{1:m-1}) - 1 \right) ds d \rva^{1:m-1}
\end{align*}

Note that $\pi_{\theta_k^m} \propto \exp \{{\theta_k^m}^\top \phi\}$, he optimality condition gives
$$\hat{\pi}^m_{k+1}(\cdot| s, \rva^{1:m-1}) \propto \exp \{ \beta_k^{-1} \hat{Q}_{\pi_k}^{1:m}(s, \rva^{1:m-1}, \cdot) + {\theta_k^m}^\top \phi(s, \rva^{1:m-1},a) \}.
$$
\end{proof}

\begin{lemma}
\label{lem:stepwise}
Suppose for any agent $m \in \gN$, policy improvement error and policy evaluation errors satisfy
\begin{gather}
\label{eq:errors1}
    \E_{\sigma_k} \Big( {\theta_{k+1}^m}^\top \phi(\sa,  a^m)-  (\beta_k^{-1} \hat{Q}_{\bpi_\thetak}^{1:m}(\sa, a^m) + {\theta_k^m}^\top \phi(\sa, a^m)) \Big)^2  \le (\epsilon_{k}^m)^2, \\
    \E_{\sigma_k} \left( \hat{Q}_{\bpi_\thetak}^{1:m}(\sa, a^m) - Q_{\bpi_\thetak}^{1:m}(\sa, a^m) \right)^2 \le (\xi_k^m)^2. \label{eq:errors2}
\end{gather}

Considering the $L_\infty$-norm of ${\theta_{k+1}^m}^\top \phi - {\theta_{k}^m}^\top \phi$ we have
$$
    \E_{s \sim \nu_*, \rva^{1:m-1} \sim \bpi_*} \left\| (\theta_{k+1}^m- \theta_k^m)^\top \phi(\sa, \cdot) -\beta_k^{-1} \hat{Q}_{\bpi_\thetak}^{1:m} (\sa, \cdot)\right\|_\infty \le \frac{\delta_k^m}{2}
$$
where $\delta_k^m = 2 \phi_k^{m-1} \epsilon_{k}^m$.
\end{lemma}
\begin{proof}
% Using the triangle inequality, we have
% \begin{align*}
%     &\quad \left\| (\theta_{k+1}- \theta_k)^\top \phi(\sa, \cdot)\right\|_\infty\\
%     &\le  \left\| (\theta_{k+1}- \theta_k)^\top \phi(\sa, \cdot) - \beta_k^{-1} \hat{Q}_{\bpi_k}^{1:m}(\sa,\cdot )\right\|_\infty +  \left\| \beta_k^{-1} \hat{Q}_{\bpi_k}^{1:m}(\sa,\cdot )\right\|_\infty
% \end{align*}

% For the first term on the RHS,
The proof is straightforward,
\begin{align*}
    &\quad \E_{s \sim \nu_*, \rva^{1:m-1} \sim \bpi_*^{1:m-1}} \left\| (\theta_{k+1}^m- \theta_k^m)^\top \phi(\sa, \cdot) - \beta_k^{-1} \hat{Q}_{\bpi_\thetak}^{1:m}(\sa,\cdot )\right\|_\infty \\
    &\le  \mathop{\E}\limits_{\substack{s \sim \nu_*\\ \rva^{1:m-1} \sim \bpi_*}} \left\| (\theta_{k+1}^m- \theta_k^m)^\top \phi(\sa, a^m) - \beta_k^{-1} \hat{Q}_{\bpi_\thetak}^{1:m}(\sa,a^m )\right\|. 
    \end{align*}
We shift from $\sigma_*$ to $\sigma_k$ and introduce concentrability coefficients to measure distributional shift
\begin{align*}
    &\quad  \mathop{\E}\limits_{\substack{s \sim \nu_k\\ \rva^{1:m-1} \sim \bpi_\thetak}} \left\| (\theta_{k+1}^m- \theta_k^m)^\top \phi(\sa, a^m) - \beta_k^{-1} \hat{Q}_{\bpi_\thetak}^{1:m}(\sa,a^m )\right\| \cdot \frac{\nu_* \bpi_*^{1:m-1}}{\nu_k \bpi_\thetak^{1:m-1}}\\
    &\le \left[ \E_{\sigma_k} \left( (\theta_{k+1}^m- \theta_k^m)^\top \phi - \beta_k^{-1} \hat{Q}_{\bpi_{\theta_k}} \right)^2 \right]^{1/2} \cdot \left[ \E_{\sigma_k} \left|   \frac{d (\nu_* \bpi_*^{1:m-1} )}{d (\nu_k \bpi_\thetak^{1:m-1}) } \right|^2  \right]^{1/2}\\
    &= \epsilon_{k}^m \phi_k^{m-1}
\end{align*}
where we use Cauchy-Schwartz inequality in the second line.

% For the second term on the RHS,
% \begin{align*}
%     &\quad \E_{s \sim \nu_*, \rva^{1:m-1} \sim \bpi_*^{1:m-1}} \left\| \beta_k^{-1} \hat{Q}_{\bpi_k}^{1:m}(\sa,\cdot )\right\|_\infty\\
%     &\le \E_{s \sim \nu_*, \rva^{1:m-1} \sim \bpi_*^{1:m-1}} \left( \left\| \beta_k^{-1} (\hat{Q}_{\bpi_k}^{1:m}(\sa,\cdot ) - Q_{\bpi_k}^{1:m}(\sa,\cdot ) )\right\|_\infty + \left\| \beta_k^{-1} Q_{\bpi_k}^{1:m}(\sa,\cdot )\right\|_\infty \right)\\
%     &\le \frac{ |\gA|}{\beta_k} \mathop{\E}\limits_{\substack{s \sim \nu_*\\ \rva^{1:m-1} \sim \bpi_*}} \mathop{\E}\limits_{a^m \sim \pi_0^m} 
%     \left\| \hat{Q}_{\bpi_{\theta_k}}^{1:m}(\sa, a^m) - Q_{\bpi_{\theta_k}}^{1:m}(\sa, a^m) \right\| + \frac{1}{\beta_k (1-\gamma)} \\
%     &= \frac{ |\gA|}{\beta_k} \mathop{\E}\limits_{\substack{s \sim \nu_k\\ \rva^{1:m-1} \sim \bpi_{\theta_k}}} \mathop{\E}\limits_{a^m \sim \pi_{\theta_k^m}}
%     \left\| \hat{Q}_{\bpi_{\theta_k}}^{1:m}(\sa, a^m) - Q_{\bpi_{\theta_k}}^{1:m}(\sa, a^m) \right\| \frac{d(\nu_* \bpi_*^{1:m-1} \pi_0^m)}{d(\nu_k \bpi_{\theta_k}^{1:m})} + \frac{1}{\beta_k(1-\gamma)} \\
%     &\le \frac{ |\gA|}{\beta_k} \left[ \E_{\sigma_k} \left( \hat{Q}_{\bpi_{\theta_k}} -  Q_{\bpi_{\theta_k}} \right)^2 \right]^{1/2} \cdot \left[ \E_{\sigma_k} \left| \frac{d(\nu_* \bpi_*^{1:m-1} \pi_0^m)}{d(\nu_k \bpi_{\theta_k}^{1:m})} \right|^2 \right]^{1/2} + \frac{1}{\beta_k(1-\gamma)}\\
%     &= \frac{ |\gA|}{\beta_k} \sqrt{\xi_{k}^m} \chi_k^m + \frac{1}{\beta_k(1-\gamma)}
% \end{align*}
The proof is completed.
\end{proof}

Recall that we define $\hat{\pi}_{k+1}^m$ as the ideal update policy based on $\hat{Q}_{\bpi_\thetak}^{1:m}$. Correspondingly, we define the ideal update based on the exact value function $Q_{\bpi_\thetak}^{1:m}$ as 
\begin{align*}
    \pi^m_{k+1} &\xleftarrow{}
     \argmax_{\pi^m} F(\pi^m)\\
     F(\pi^m) &= \E_{\sigma_k} \Big[\langle Q_{\bpi_\thetak}^{1:m}(s, \rva^{1:m-1}, \cdot), \pi^m (\cdot|s, \rva^{1:m-1})\rangle -\beta_k KL\left(\pi^m(\cdot|s, \rva^{1:m-1})\| \pi_{\theta_k^m}(\cdot|s, \rva^{1:m-1})\right) \Big].
\end{align*}

Under log-linear parametrization: $\pi_{\theta_k^m} \propto \exp \{ \phi^\top \theta_k^m\}$, analogously we have
\begin{align*}
    \pi^m_{k+1} &(\cdot| s, \rva^{1:m-1}) \propto \exp \left\{ \beta_k^{-1} Q_{\pi_k}^{1:m}(s, \rva^{1:m-1}, \cdot) + \phi^\top(s, \rva^{1:m-1}, \cdot)  \theta_k^m \right\}.
\end{align*}
\begin{lemma}[Error Propagation]
\label{lem:err} 
Suppose for any agent $m \in \gN$ and $(\sa) \in \gS \times \gA^{m-1}$, policy improvement and policy evaluation errors satisfy,
\begin{gather*}
     \E_{\sigma_k} \Big( (\theta_{k+1}^m- \theta_k^m)^\top \phi- \beta_k^{-1} \hat{Q}_{\bpi_\thetak}^{1:m}  \Big)^2  \le (\epsilon_{k}^m)^2,\\
    \E_{\sigma_k} \left( \hat{Q}_{\bpi_\thetak}^{1:m} - Q_{\bpi_\thetak}^{1:m} \right)^2  \le (\xi_k^m)^2
\end{gather*}
where we omit $(\sa, a^m)$ for simplicity.

Compare the statistical error, we have
\begin{align*}
    \Bigg| \E_{s \sim \nu_*, \rva \sim \bpi_*} &\Bigg\langle \log{\frac{\pi_{\theta_{k+1}^m}(\cdot|\sa)}{\pi_{k+1}^m(\cdot|\sa)}}, \pi_{*}^m(\cdot|\sa)
    - \pi_{\theta_k^m}(\cdot|\sa) \Bigg\rangle \Bigg| \le \Delta_k^m.
\end{align*}
% where $\Delta_k^m = \phi_{k}^m \epsilon_{k}^m  + \frac{1}{\beta_k} \psi_k^m \xi^m_k.$
where $\Delta_k^m =\sqrt{2}(\phi_{k}^m + \phi^{m-1}_k) \cdot \left(\epsilon_k^m + \frac{\xi_k^m}{\beta_k} \right)$
\end{lemma}

Lemma~\ref{lem:err} presents the quantitative differences between the actual parameterized $\pi_{\theta_{k+1}^m}$ based on $\hat{Q}^{1:m}$  and the ideal policy $\pi_{k+1}^m$ based on the exact value function $Q^{1:m}$.

% \paragraph{Proof for Lemma~\ref{lem:err}.} Under the same conditions~\eqref{eq:errors1},~\eqref{eq:errors2} as Lemma~\ref{lem:stepwise}, we provide the following argument
% \begin{align*}
%     \Bigg| \E_{s \sim \nu_*, \rva \sim \bpi_*} &\Bigg\langle \log{\frac{\pi_{\theta_{k+1}^m}(\cdot|\sa)}{\pi_{k+1}^m(\cdot|\sa)}}, \pi_{*}^m(\cdot|\sa)
%     - \pi_{\theta_k^m}(\cdot|\sa) \Bigg\rangle \Bigg| \le \Delta_k^m.
% \end{align*}

\begin{proof}
First, from definition for any $m \in \gN$ and $\sa \in \gS \times \gA^{m-1}$ we have
\begin{gather}
    \pi_{k+1}^m(\cdot | \sa) = \exp \left\{{\beta_k}^{-1} Q_{\bpi_{\theta_k}}+ {\theta_k^m}^\top \phi(\sa, \cdot) \right\} / W(\sa), \label{eq:policy1}\\
    \pi_{\theta_{k+1}^m}(\cdot | \sa) = \exp \left\{ {\theta_{k+1}^m}^\top \phi(\sa, \cdot) \right\} / M(\sa). \label{eq:policy2}
\end{gather}
Substituting this into the expression, we have
\begin{align*}
    \text{LHS} &= \left\langle \log{\pi_{\theta_{k+1}^m}} - \log{\pi_{k+1}^m}, \pi_{*}^m
    - \pi_{\theta_k^m} \right\rangle\\
    &= \left \langle {\theta_{k+1}^m}^\top \phi - (\beta_k^{-1} Q_{\bpi_{\theta_k}} + {\theta_k^m}^\top \phi), \pi_{*}^m  - \pi_{\theta_k^m} \right\rangle\\
    &= \underbrace{\left \langle {\theta_{k+1}^m}^\top \phi - (\beta_k^{-1} \hat{Q}_{\bpi_{\theta_k}} + {\theta_k^m}^\top \phi), \pi_{*}^m - \pi_{\theta_k^m} \right\rangle}_{(a)} + \underbrace{\left \langle \beta_k^{-1} \hat{Q}_{\bpi_{\theta_k}} - \beta_k^{-1} Q_{\bpi_{\theta_k}} , \pi_{*}^m - \pi_{\theta_k^m} \right\rangle}_{(b)}
\end{align*}
In the second line,  we use the fact that $ \left \langle \log{\frac{W}{M}}, \pi_{*}^m - \pi_{\theta_k^m} \right\rangle = \log{\frac{W}{M}} \sum_{a^m} \left(  \bpi_*^m(\cdot) -  \pi_{\theta_k^m}(\cdot) \right) = 0$.

Bounding the two terms separately, we have
\begin{itemize}
    \item For $(a)$, taking the expectation over $\gS \times \gA^{m-1}$ we have
    \begin{align*}
     &\quad \left| \E_{s \sim \nu_*} \E_{\rva^{1:m-1} \sim \bpi_*} \left \langle {\theta_{k+1}^m}^\top \phi - (\beta_k^{-1} \hat{Q}_{\bpi_{\theta_k}} + {\theta_k^m}^\top \phi), \pi_{*}^m - \pi_{\theta_k^m} \right\rangle \right| \\
    &= \left| \mathop{\int}\limits_{\gS \times \gA^{m-1} \times \gA } \left[ (\theta_{k+1}^m- \theta_k^m)^\top \phi - \beta_k^{-1} \hat{Q}_{\bpi_{\theta_k}} \right] \cdot \left( \pi_{*}^m - \pi_{\theta_k^m} \right)  d a^m \cdot \bpi_*^{1:m-1} d (\rva^{1:m-1}) \cdot \nu_*(s) ds\right|
\end{align*}
 Here the expectation is taken w.r.t. $\sigma*$, we change it to be expectation over $\sigma_k$ by introducing concentrability coefficients
\begin{align*}
    &\quad \left| \mathop{\int}\limits_{\gS \times \gA^{m-1} \times \gA } \left[ (\theta_{k+1}^m- \theta_k^m)^\top \phi - \beta_k^{-1} \hat{Q}_{\bpi_{\theta_k}} \right] \cdot \left( \frac{\bpi_*^{1:m}}{\bpi_\thetak^{1:m}} -  \frac{\bpi_*^{1:m-1} \pi_{\theta_k^m}}{\bpi_\thetak^{1:m}} \right)  \bpi_\thetak^{1:m}(\rva^{1:m}|s)  d (\rva^{1:m}) \cdot \nu_*(s) ds\right|  \\
    % &= \left| \mathop{\int}\limits_{\gS \times \gA^{m-1} \times \gA } \left[ (\theta_{k+1}^m- \theta_k^m)^\top \phi - \beta_k^{-1} \hat{Q}_{\bpi_{\theta_k}} \right] \cdot \frac{\nu_*(s)}{\nu_k(s)} \left( \frac{\bpi_*^{1:m}}{\bpi_\thetak^{1:m}} -  \frac{\bpi_*^{1:m-1} \pi_{\theta_k^m}}{\bpi_\thetak^{1:m}} \right)  \bpi_\thetak^{1:m}(\rva^{1:m}|s)  d (\rva^{1:m}) \cdot \nu_k(s) ds\right|  \\
    &= \left| \mathop{\int}\limits_{\gS \times \gA^{m-1} \times \gA } \left[ (\theta_{k+1}^m- \theta_k^m)^\top \phi - \beta_k^{-1} \hat{Q}_{\bpi_{\theta_k}} \right] \cdot \frac{\nu_*(s)}{\nu_k(s)} \left( \frac{\bpi_*^{1:m}}{\bpi_\thetak^{1:m}} -  \frac{\bpi_*^{1:m-1} \pi_{\theta_k^m}}{\bpi_\thetak^{1:m}} \right)  d \sigma_k \right|  \\
    &\stackrel{(i)}{\le} \left[ \E_{\sigma_k} \left( (\theta_{k+1}^m- \theta_k^m)^\top \phi - \beta_k^{-1} \hat{Q}_{\bpi_{\theta_k}} \right)^2 \right]^{1/2} \cdot \left[ \E_{\sigma_k} \left|   \frac{ d( \nu_*\bpi_*^{1:m})}{d (\nu_k \bpi_\thetak^{1:m}) } - \frac{d (\nu_*\bpi_*^{1:m-1} ) }{d (\nu_k \bpi_\thetak^{1:m-1})} \right|^2  \right]^{1/2}\\
    &\stackrel{(ii)}{\le} \sqrt{2} \epsilon_{k}^m (\phi_k^m + \phi_k^{m-1})
\end{align*}
$(i):$ This is because of the Cauchy-Schwartz inequality.\\
$(ii):$ This is because: 
$
\sqrt{\int |f - g|^2 d \sigma } \le \sqrt{ \int 2|f|^2 d \sigma  + \int 2 |g|^2 d \sigma} \le \sqrt{2} (\|f\|_{2, \sigma} + \|g\|_{2, \sigma}).
$

\item For $(b)$, taking the expectation over $\gS \times \gA^{m-1}$ we have
    \begin{align*}
     &\quad \left| \E_{s \sim \nu_*} \E_{\rva^{1:m-1} \sim \bpi_*} \left \langle \beta_k^{-1} \hat{Q}_{\bpi_{\theta_k}} - \beta_k^{-1} Q_{\bpi_{\theta_k}}, \pi_{*}^m - \pi_{\theta_k^m} \right\rangle \right| \\
    &= \left| \mathop{\int}\limits_{\gS \times \gA^{m-1} \times \gA } \left[ \beta_k^{-1} \hat{Q}_{\bpi_{\theta_k}} - \beta_k^{-1} Q_{\bpi_{\theta_k}} \right] \cdot \left( \pi_{*}^m - \pi_{\theta_k^m} \right)  d a^m \cdot \bpi_*^{1:m-1} d (\rva^{1:m-1}) \cdot \nu_*(s) ds\right|
    \end{align*}
Analogously we replace the expectation over $\sigma*$ with expectation over $\sigma_k$
\begin{align*}
    % &\quad \left| \mathop{\int}\limits_{\gS \times \gA^{m-1} \times \gA } \left[ \beta_k^{-1} \hat{Q}_{\bpi_{\theta_k}} - \beta_k^{-1} Q_{\bpi_{\theta_k}} \right] \cdot \left( \frac{\bpi_*^{1:m}}{\bpi_{\theta_k}^{1:m}} -  \frac{\bpi_*^{1:m-1} \pi_{\theta_k^m}}{\bpi_{\theta_k}^{1:m}} \right)  \bpi_{\theta_k}^{1:m}(\rva^{1:m}|s)  d (\rva^{1:m}) \cdot \nu_*(s) ds\right|  \\
    &\quad \left| \mathop{\int}\limits_{\gS \times \gA^{m-1} \times \gA } \left[ \beta_k^{-1} \hat{Q}_{\bpi_{\theta_k}} - \beta_k^{-1} Q_{\bpi_{\theta_k}} \right] \cdot \frac{\nu_*(s)}{\nu_k(s)} \left( \frac{\bpi_*^{1:m}}{\bpi_{\theta_k}^{1:m}} -  \frac{\bpi_*^{1:m-1} \pi_{\theta_k^m}}{\bpi_{\theta_k}^{1:m}} \right)  \bpi_{\theta_k}^{1:m}(\rva^{1:m}|s)  d (\rva^{1:m}) \cdot \nu_k(s) ds\right|  \\
    &= \left| \mathop{\int}\limits_{\gS \times \gA^{m-1} \times \gA } \left[ \beta_k^{-1} \hat{Q}_{\bpi_{\theta_k}} - \beta_k^{-1} Q_{\bpi_{\theta_k}} \right] \cdot \frac{\nu_*(s)}{\nu_k(s)} \left( \frac{\bpi_*^{1:m}}{\bpi_{\theta_k}^{1:m}} -  \frac{\bpi_*^{1:m-1} }{\bpi_{\theta_k}^{1:m-1}} \right)  d \sigma_k \right|  \\
    &\stackrel{(i)}{\le} \left[ \E_{\sigma_k} \left( \beta_k^{-1} \hat{Q}_{\bpi_{\theta_k}} - \beta_k^{-1} Q_{\bpi_{\theta_k}} \right)^2 \right]^{1/2} \cdot \left[ \E_{\sigma_k} \left| 
    \frac{d ( \nu_* \bpi_*^{1:m}) }{d ( \nu_k \bpi_{\theta_k}^{1:m}) } -
     \frac{d( \nu_* \bpi_*^{1:m-1}) }{ d( \nu_k \bpi_{\theta_k}^{1:m-1}) }  \right|^2 \right]^{1/2}\\
    &\stackrel{(ii)}{\le} \frac{\sqrt{2}}{\beta_k} \xi_{k}^m (\phi_k^m + \phi_k^{m-1})
\end{align*}
$(i):$ This is because of the Cauchy-Schwartz inequality.\\
$(ii):$ This is because: 
$
\sqrt{\int |f - g|^2 d \sigma } \le \sqrt{ \int 2|f|^2 d \sigma  + \int 2 |g|^2 d \sigma} \le \sqrt{2} (\|f\|_{2, \sigma} + \|g\|_{2, \sigma}).
$
\end{itemize}
Combining the bounds, we conclude the proof for Lemma~\ref{lem:err}
$$
\Delta_k^m = \sqrt{2}(\phi_{k}^m + \phi^{m-1}_k) \cdot \left(\epsilon_k^m + \frac{\xi_k^m}{\beta_k} \right).
$$
\end{proof}

Below we introduce a lemma that is crucial in our multi-agent PPO analysis. The original version is widely found and proven to be useful for mirror descent analysis~\citep{nesterov2003introductory}.
\begin{lemma}[One-Step Descent]
\label{lem:one_des}
For the ideal updated policy $\pi_{k+1}^m$, the real updated policy $\pi_{\theta_{k+1}^m}$ and current policy $\pi_{\theta_{k}^m}$, we have that for any $(s, \rva^{1:m-1}) \in \gS \times \gA^{m-1}$,
\begin{align*}
&\quad KL\left(\pi_*^m(\cdot | \sa)\| \pi_{\theta_{k+1}^m}(\cdot| \sa) \right) - KL\left(\pi_*^m(\cdot | \sa)\| \pi_{\theta_{k}^m}(\cdot| \sa) \right)\\
&\le \left\langle \log{\frac{\pi_{\theta_{k+1}^m}(\cdot|\sa)}{\pi_{k+1}^m(\cdot|\sa)}}, \pi_{\theta_{k}^m}(\cdot|\sa) - \pi_{*}^m(\cdot|\sa) \right\rangle\\
&+ \frac{1}{\beta_k} \left \langle A_{\pi_{\theta_k}}^m(\sa, \cdot), \pi_{\theta_{k}^m}(\cdot|\sa) - \pi_{*}^m(\cdot|\sa) \right\rangle - \frac{1}{2} \left\| \pi_{\theta_{k+1}^m}(\cdot|\sa) - \pi_{\theta_k^m}(\cdot|\sa)\right\|_1^2\\
&- \left \langle (\theta_{k+1}^m - \theta_k^m)^\top \phi(\sa, \cdot), \pi_{\theta_k^m}(\cdot|\sa) - \pi_{\theta_{k+1}^m}(\cdot|\sa) \right\rangle
\end{align*}
\end{lemma}

\begin{proof}
In the proof, we simply omit $(\sa)$ when making no abuse of notation.

Using the definition, we have
\begin{align}
     &\quad KL\left(\pi_*^m \| \pi_{\theta_k^m} \right) - KL\left(\pi_*^m \| \pi_{\theta_{k+1}^m} \right) \notag \\
     &= \left\langle \log \frac{\pi_{\theta_{k+1}^m}}{\pi_{\theta_k^m}}, \pi_*^m \right \rangle \notag \\
     &= \left\langle \log \frac{\pi_{\theta_{k+1}^m}}{\pi_{\theta_k^m}}, \pi_*^m - \pi_{\theta_{k+1}^m} \right \rangle + KL\left(\pi_{\theta_{k+1}^m} \| \pi_{\theta_k^m} \right) \label{eq:KLgap}
\end{align}
Recall the definitions of $\pi_{k+1}^m$~\eqref{eq:policy1} and $\pi_{\theta_{k+1}^m}$~\eqref{eq:policy2}, we have the following two equations 
\begin{align}
    \left\langle \log \frac{\pi_{\theta_{k+1}^m}}{\pi_{\theta_k^m}}, \pi_{\theta_k^m} - \pi_{\theta_{k+1}^m} \right \rangle &= \left \langle (\theta_{k+1}^m-\theta_k^m)^\top \phi, \pi_{\theta_k^m} - \pi_{\theta_{k+1}^m} \right \rangle, \label{eq:fact1}\\
    \left\langle  \beta_k^{-1} Q_{\pi_\thetak}, \pi_*^m - \pi_{\theta_k^m} \right \rangle &= \left\langle\log{ \frac{\pi_{k+1}^m}{\pi_{\theta_k^m}} }, \pi_*^m - \pi_{\theta_k^m} \right \rangle. \label{eq:fact2}
\end{align}
Plugging these results~\eqref{eq:fact1}, ~\eqref{eq:fact2} into the RHS of~\eqref{eq:KLgap} we have
\begin{align*}
    &\quad KL\left(\pi_*^m \| \pi_{\theta_k^m} \right) - KL\left(\pi_*^m \| \pi_{\theta_{k+1}^m} \right) \\
    &= \left\langle \log \frac{\pi_{\theta_{k+1}^m}}{\pi_{k+1}^m}, \pi_*^m - \pi_{\theta_k^m} \right \rangle +
    \left\langle \beta_k^{-1} Q_{\pi_\thetak}^{1:m}, \pi_*^m - \pi_{\theta_k^m} \right \rangle + \left\langle (\theta_{k+1}^m - \theta_k^m)^\top \phi, \pi_{\theta_k^m} - \pi_{\theta_{k+1}^m} \right \rangle + KL\left(\pi_{\theta_{k+1}^m} \| \pi_{\theta_k^m} \right)
    \\
    &\ge \left\langle \log \frac{\pi_{\theta_{k+1}^m}}{\pi_{k+1}^m}, \pi_*^m - \pi_{\theta_k^m} \right \rangle +
    \left\langle \beta_k^{-1} A_{\pi_\thetak}^{m}, \pi_*^m - \pi_{\theta_k^m} \right \rangle + \left\langle (\theta_{k+1}^m - \theta_k^m)^\top \phi, \pi_{\theta_k^m} - \pi_{\theta_{k+1}^m} \right \rangle + \frac{1}{2} \left\| \pi_{\theta_{k+1}^m} - \pi_{\theta_k^m} \right\|_1^2
\end{align*}
In the last line: (1) From the Definition~\ref{def::1} of multi-agent advantage functions, we have 
$$
\langle Q^{1:m}_{\pi_\thetak} - A_{\pi_\thetak}^m, \pi_*^m - \pi_{\theta_k^m} \rangle =0,
$$
and (2) Pinsker's inequality in information theory gives a lower bound of the $KL$-divergence. 
$$
KL\left(\pi_{\theta_{k+1}^m} \| \pi_{\theta_k^m} \right) \ge \frac{1}{2} \left\| \pi_{\theta_{k+1}^m} - \pi_{\theta_k^m} \right\|_1^2,
$$plugging these into the expression, 
which concludes the proof.
\end{proof}

With these results, we are ready to present the proofs for the main theorem.
\paragraph{Proofs for Theorem~\ref{thm_mappo}.}
\begin{proof}
With Lemma~\ref{lem:one_des}, take expectation with respect to $s \sim \nu_*$ and $\rva \sim \bpi_*$, we have
\begin{align*}
&  \frac{1}{\beta_k} \mathop{\E}\limits_{\sigma_*} \left \langle A^m_{\pi_{\thetak}}(\sa, \cdot), \pi_{*}^m(\cdot|\sa) - \pi_{\theta_k^m}(\cdot|\sa) \right\rangle\\
&\le \mathop{\E}\limits_{\sigma_*} \left[ KL\left(\pi_*^m(\cdot | \sa)\| \pi_{\theta_{k}^m}(\cdot| \sa) \right) -  KL\left(\pi_*^m(\cdot | \sa)\| \pi_{\theta_{k+1}^m}(\cdot| \sa) \right)  \right] \\
&- \mathop{\E}\limits_{\sigma_*} \left\langle \log{\frac{\pi_{\theta_{k+1}^m}(\cdot|\sa)}{\pi_{k+1}^m(\cdot|\sa)}}, \pi_{*}^m(\cdot|\sa) - \pi_{\theta_k^m}(\cdot|\sa) \right\rangle\\
&- \frac{1}{2} \mathop{\E}\limits_{\sigma_*} \left\| \pi_{\theta_{k+1}^m}(\cdot|\sa) - \pi_{\theta_k^m}(\cdot|\sa)\right\|_1^2\\
&- \mathop{\E}\limits_{\sigma_*} \left \langle {\theta_{k+1}^m}^\top \phi(\sa, \cdot) - {\theta_k^m}^\top \phi(\sa, \cdot), \pi_{\theta_k^m}(\cdot|\sa) - \pi_{\theta_{k+1}^m}(\cdot|\sa) \right\rangle.
\end{align*}

Analogous to the previous section, we omit $(\sa)$ below for simplicity when making no abuse of notation. We arrange the above inequality by plugging in Lemma~\ref{lem:err}
\begin{align}
  &\quad \frac{1}{\beta_k} \mathop{\E}\limits_{\sigma_*} \left \langle A^m_{\pi_{\thetak}}, \pi_{*}^m - \pi_{\theta_k^m} \right\rangle \notag \\
&\le \mathop{\E}\limits_{\sigma_*} \left[ KL\left(\pi_*^m\| \pi_{\theta_{k}^m} \right) -  KL\left(\pi_*^m\| \pi_{\theta_{k+1}^m} \right) \right] \notag \\
&\qquad 
\qquad - \mathop{\E}\limits_{\sigma_*} \left[ \left\langle \log{\frac{\pi_{\theta_{k+1}^m}}{\pi_{k+1}^m}}, \pi_{*}^m - \pi_{\theta_k^m} \right\rangle
+ \frac{1}{2} \left\| \pi_{\theta_{k+1}^m} - \pi_{\theta_k^m} \right\|_1^2
+  \left \langle (\theta_{k+1}^m-\theta_k^m)^\top \phi, \pi_{\theta_k^m} - \pi_{\theta_{k+1}^m} \right\rangle  \right] \notag \\
&\le  \underbrace{\mathop{\E}\limits_{\sigma_*} \left[  KL\left(\pi_*^m\| \pi_{\theta_{k}^m} \right) -  KL\left(\pi_*^m\| \pi_{\theta_{k+1}^m} \right)\right]}_{(i)}    + \Delta_k^m
\underbrace{- \mathop{\E}\limits_{\sigma_*} \left[ 
\frac{1}{2} \left\| \pi_{\theta_{k+1}^m} - \pi_{\theta_k^m} \right\|_1^2+
  \left \langle (\theta_{k+1}^m-\theta_k^m)^\top \phi, \pi_{\theta_k^m} - \pi_{\theta_{k+1}^m} \right\rangle \right]}_{(ii)}  \label{eq:pre-sum}
\end{align}

Take summation over $m = 1, \cdots, N$ and $k=0, \cdots, K-1$ for both sides, note that

\paragraph{LHS} This equals $(1-\gamma) \sum_{k =0}^{K-1} \frac{1}{\beta_k} (J(\bpi_*) - J(\bpi_\thetak))$ by performance difference lemma (cf. Lemma~\ref{lem:perf}).

\paragraph{RHS}
$(i):$ After taking summation over $k$ and $m$, this term is upper bounded by $N \log{\gA}$ because for any $m \in \gN$, we have $\mathop{\E}\limits_{\sigma_*}KL\left(\pi_*^m\| \pi_{\theta_{0}}^m \right) \le \log|\gA|$ since the initial policy $\pi_{\theta_0}$ is uniformly distributed over action spaces. 

$(ii)$: Using H\"older inequality, we have
\begin{align*}
    \quad - \mathop{\E}\limits_{\sigma_*} 
  \left \langle (\theta_{k+1}^m-\theta_k^m)^\top \phi, \pi_{\theta_k^m} - \pi_{\theta_{k+1}^m} \right\rangle \le \mathop{\E}\limits_{\sigma_*} \left\| (\theta_{k+1}^m-\theta_k^m)^\top \phi \right\|_\infty \cdot \left\|\pi_{\theta_k^m} - \pi_{\theta_{k+1}^m}\right\|_1
  \end{align*}
  Using the triangle inequality, we can upper bound it by
  \begin{align*}
  &\quad \mathop{\E}\limits_{\sigma_*} \left\| (\theta_{k+1}^m-\theta_k^m)^\top \phi - 
 \beta_k^{-1} \hat{Q}\right\|_\infty   \cdot \left\|\pi_{\theta_k^m} - \pi_{\theta_{k+1}^m}\right\|_1 + \mathop{\E}\limits_{\sigma_*} \left\|  
 \beta_k^{-1} \hat{Q}\right\|_\infty   \cdot \left\|\pi_{\theta_k^m} - \pi_{\theta_{k+1}^m}\right\|_1\\
  &\le \delta_k^m +  \mathop{\E}\limits_{\sigma_*} \left\|  
 \beta_k^{-1} \hat{Q}\right\|_\infty   \cdot \left\|\pi_{\theta_k^m} - \pi_{\theta_{k+1}^m}\right\|_1,
\end{align*}
where we plug in Lemma~\ref{lem:stepwise} and: $\left\|\pi_{\theta_k^m} - \pi_{\theta_{k+1}^m}\right\|_1 \le \left\|\pi_{\theta_k^m} \right\|_1 + \left\|\pi_{\theta_{k+1}^m} \right\|_1 =2$.

We have
\begin{align*}
     & \quad - \mathop{\E}\limits_{\sigma_*} \left[ 
\frac{1}{2} \left\| \pi_{\theta_{k+1}^m} - \pi_{\theta_k^m} \right\|_1^2 \right] + \mathop{\E}\limits_{\sigma_*} \left\|  
 \beta_k^{-1} \hat{Q}\right\|_\infty   \cdot \left\|\pi_{\theta_k^m} - \pi_{\theta_{k+1}^m}\right\|_1\\
 &\le  \mathop{\E}\limits_{\sigma_*} \frac{1}{2} \left\|
 \beta_k^{-1} \hat{Q}\right\|_\infty^2, \quad \text{because $\forall x, y$ it holds: $- \frac{1}{2}x^2 + y x \le \frac{1}{2} y^2$.}\\
 &\le \frac{B^2}{2 \beta_k^2}
\end{align*}

Finally, by combining these results, we rearrange ~\eqref{eq:pre-sum} and obtain
\begin{align*}
&\quad (1-\gamma) \sum_{k =0}^{K-1} \frac{1}{\beta_k} (J(\bpi_*) - J(\bpi_\thetak))
\le N \log|\gA| + \sum_{m=1}^N \sum_{k=0}^{K-1} (\Delta_k^m+\delta_k^m) + \sum_{m=1}^N \sum_{k=0}^{K-1} \frac{B^2}{2\beta_k^2}.
\end{align*}
Setting the penalty parameter $\beta_k = \beta \sqrt{K}$ and noting that $\bar{\bpi}$ is uniformly sampled from $\bpi_{\theta_k}, k = 1,2 \cdots K-1$, we have 
$$
J(\bpi_*) - J(\bar{\bpi}) \le \frac{N \beta^2 \log{|\gA|} + N B^2/2 + \beta^2 \sum_{m=1}^N \sum_{k=0}^{K-1} (\Delta_k^m + \delta_k^m)}{(1-\gamma)\beta \sqrt{K}}.
$$
Considering policy improvement/evaluation errors, the optimal choice for $\beta$ is
$$ 
\beta =  \sqrt{\frac{ NB^2/2}{N\log{|\gA|} + \sum_{m=1}^N \sum_{k=0}^{K-1} (\Delta_k^m + \delta_k^m)}  }
$$
then
$$
J(\bpi_*) - J(\bar{\bpi}) \le \gO \left( \frac{B \sqrt{N}}{1-\gamma} 
\sqrt{ 
    \frac{
        N \log{|\gA|} + \sum_{m=1}^N \sum_{k=0}^{K-1} (\Delta_k^m + \delta_k^m)
        }{K}     
} \right).
$$
The proof is completed.
\end{proof}

\section{Proofs for Section~\ref{sec:pess}}
\label{sec:pf-pess}
\subsection{Computational efficiency}
Observe the pessimistic evaluation
\[
\omega_k^m \xleftarrow[]{} \argmin_{\omega} \left(f(s_0, \bpi^{1:m}_k) + \lambda \gE^{1:m}(f, \bpi_\thetak) \right).
\] 

Under linear function approximation, $f(s_0, \bpi^{1:m})$ is instantiated as $\phi(s_0, \bpi^{1:m})^\top \omega$, then
\begin{align*}
    &L^{1:m}(f^\prime, f, \bpi)= \frac{1}{n} \sum_{\gD^m} \left( \phi(s, \rva^{1:m})^\top \omega^\prime - r - \gamma \phi(s^\prime, \bpi^{1:m})^\top \omega \right)^2,
\end{align*}
thus we have $\gE^{1:m} (f, \bpi)$ defined as
\begin{align*}
    \sum_{\gD_m} ( \phi(s, \rva^{1:m})^\top \omega - r - \gamma \phi(s^\prime, \bpi^{1:m})^\top \omega)^2 - \min_{\omega^\prime}\sum_{\gD_m} ( \phi(s, \rva^{1:m})^\top \omega^\prime - r - \gamma \phi(s^\prime, \bpi^{1:m})^\top \omega)^2,
\end{align*}
where summation is taken over samples $(s, \rva^{1:m}, r, s^\prime)$ from $\gD^m$.

Therefore, the Bellman error has a quadratic-form dependency on value function parameter $\omega$, allowing the application of many efficient numerical solvers.

\subsection{Proofs}
The linear function approximation directly implies the \textbf{Realizability} and \textbf{Completeness} conditions: For any $m \in \gN$, $\bpi \in \Pi^m$,
\begin{align*}
     \inf_{f \in \gF^m} \sup_{\text{admissable } \nu} \|f - \gT_{\bpi}^{1:m} f\|_{2, \nu}^2 = 0,
\end{align*}
 and
\begin{align*}
      \sup_{f \in \gF^m} \inf_{f \in \gF^m} \|f^\prime - \gT_{\bpi}^{1:m} f\|_{2, \mu}^2  = 0.
\end{align*}
These conditions hold because we assume a linear structure for state-action value functions: $Q_\bpi^{1:m} \in \gF^m$ (cf. Definition~\ref{def:linear}).

First we examine concentration analysis for linear function approximation~\citep{xie2021bellman}.
\begin{lemma}
\label{lem:c1}
For any $m \in \gN$, $\bpi \in \Pi^m$, with probability at least $1-\delta$ it holds
$$
\gE^{1:m}(Q^\bpi, \pi) \le \gO \left( \frac{ d \log{\frac{n L R}{ \delta}}}{(1-\gamma)^2 n} \right) = \varepsilon_r.
$$
\end{lemma}

\begin{lemma}
\label{lem:c2}
For any $m \in \gN$, $\bpi \in \Pi^m$, $f \in \gF^m$ (cf. Definition~\ref{def:linear}) , if $\gE^{1:m}(f, \bpi) \le \varepsilon$,  with probability at least $1-\delta$ it holds
$$
\left\| f - \gT_\bpi^{1:m} f \right\|_{2, \gD^m} \le  \gO \left( \sqrt{ \frac{ d \log{\frac{n L R}{ \delta}}}{(1-\gamma)^2 n} } \right) + \sqrt{\varepsilon}.
$$
\end{lemma} 

 For simplicity, below, we shall define 
$$
R(s, \rva^{1:m}) = \E_{\Tilde{\rva} \sim \Tilde{\bpi}_k} r(s, \rva^{1:m}, \Tilde{\rva}).
$$

In the following lemma, we show that at every iteration $k$ of Algorithm~\ref{alg:2}, there exists a Markov game $\gM_k$ whose multi-agent value function is exactly $f_k^m$, $m \in \gN$. Moreover, the transition dynamics of $\gM_k$ are the same as those of the original $\gM$. We have the following theoretical guarantees to control the differences between the reward of $\gM$ and rewards of $\gM_k$.

\begin{lemma}
\label{lem:c3}
At each iteration $k$, there exists a Markov game $\gM_k$ that has the same dynamics as original $\gM$. Let the reward function of $\gM_k$ be $R_k$, then
$$
\left\|R_k^{1:m}(s, \rva^{1:m}) - R(s, \rva^{1:m}) \right\|^2_{2, \gD^m} \le \varepsilon_r.
$$

\end{lemma}

\begin{proof} Set $R_k = (I - \gamma \gP) f_k$. It directly implies that
\begin{align*}
    f_k^m &= \gT_{\bpi_k, \gM_k}^{1:m}(s. \rva^{1:m})\\
    &= R_k^{1:m}(s, \rva^{1:m}) + \gamma\E f_k^m(s^\prime, \rva^\prime).
\end{align*}
Therefore 
$$
\left\|R_k^{1:m}(s, \rva^{1:m}) - R(s, \rva^{1:m}) \right\|^2_{2, \gD^m} = \|f_k^m - \gT_{\bpi_k}^{1:m} f_k^m\|_{2, \gD^m}^2 \le \varepsilon_r$$

\end{proof}

\begin{lemma}
\label{lem:c4}
For any conditional policy $\pi : \gS \times \gA^{m-1} \xrightarrow{} \Delta(\gA)$, and $s \in \gS, \rva^{1:m-1} \in \gA^{m-1}$,
\begin{align*}
    &\quad \sum_{k=1}^K \left \langle \pi_{k+1}(\cdot | \sa), f_k^m(\sa, \cdot) \right \rangle - \ell_{sa}(\pi_1(\cdot| \sa))\\ &\ge \sum_{k=1}^K \left \langle \pi(\cdot| \sa), f_k^m(\sa, \cdot) \right \rangle - \ell_{sa}(\pi)
\end{align*}
where we assume $\ell_{sa}(\pi) = \frac{1}{\eta} \sum_{a \in \gA} \pi(a| \sa) \cdot \log{\pi(a| \sa)}$
\end{lemma}
\begin{proof}
Proofs are straightforward by noticing the fact that $\theta_{k+1}^m = \theta_k^m + \eta w_k^m$.
\end{proof} 

\begin{lemma}
\label{lem:c5}
For any conditional policy $\pi : \gS \times \gA^{m-1} \xrightarrow{} \Delta(\gA)$, and $s \in \gS, \rva^{1:m-1} \in \gA^{m-1}$,
\begin{align*}
    &\quad \sum_{k=1}^K \left \langle \pi(\cdot | \sa)- \pi_{k}(\cdot | \sa), f_k^m(\sa, \cdot) \right \rangle\\
    &\le \sum_{k=1}^K \left \langle \pi_{k+1}(\cdot | \sa)- \pi_{k}(\cdot | \sa), f_k^m(\sa, \cdot) \right \rangle - \ell_{sa}(\pi_1(\cdot| \sa))
\end{align*}
\end{lemma}
\begin{proof}
The results could be obtained by applying Lemma~\ref{lem:c4}.
\end{proof}

\begin{lemma}
\label{lem:c6}
For any conditional policy $\pi : \gS \times \gA^{m-1} \xrightarrow{} \Delta(\gA)$, and $s \in \gS, \rva^{1:m-1} \in \gA^{m-1}$, if set stepsize $\eta = \sqrt{\frac{\log{|\gA|}}{2 K}}$
\begin{align*}
    &\quad \sum_{k=1}^K \left \langle \pi(\cdot | \sa)- \pi_{k}(\cdot | \sa), f_k^m(\sa, \cdot) \right \rangle \le 2 \sqrt{2 \log{|\gA|} K}
\end{align*}
\end{lemma}

\begin{proof}
Define $\gL_{sa, k}^m = \sum_{k^\prime=1}^k \langle \pi(\cdot| \sa), f_{k^\prime}(\sa, \cdot) \rangle - \ell_{sa}(\pi)$. Let $B_{\gL_{sa, k}^m} (\cdot \|\cdot)$ and $B_{\ell_{sa}}(\cdot \| \cdot)$ be the Bregman divergences w.r.t. losses $\gL_{sa, k}^m$ and $\ell_{sa}$. Using the property of divergence we have
\begin{align*}
    \gL_{sa, k}^m(\pi_k^m) &\le \gL_{sa, k}^m(\pi_{k+1}^m) + B_{\gL_{sa, k}^m} (\pi_k^m(\cdot| \sa) \|\pi_{k+1}^m(\cdot| \sa))\\
    &= \gL_{sa, k}^m(\pi_{k+1}^m) - B_{\ell_{sa}} (\pi_k^m(\cdot| \sa) \|\pi_{k+1}^m(\cdot| \sa)).
\end{align*}
Reordering it we have
\begin{equation*}
    B_{\ell_{sa}} (\pi_k^m(\cdot| \sa) \|\pi_{k+1}^m(\cdot| \sa)) \le \gL_{sa, k}^m(\pi_{k+1}^m) - \gL_{sa, k}^m(\pi_{k}^m).
\end{equation*}
RHS of the expression above is not greater than
\begin{equation*}
    \left \langle \pi_{k+1}^m(\cdot|\sa) - \pi_{k}^m(\cdot|\sa), f_k^m(\sa, \cdot) \right \rangle.
\end{equation*}
Then
\begin{align*}
    &\quad \left \langle \pi_{k+1}^m(\cdot|\sa) - \pi_{k}^m(\cdot|\sa), f_k^m(\sa, \cdot) \right \rangle \\
    &\le \sqrt{2 \eta B_{\ell_{sa}} (\pi_k^m(\cdot| \sa) \|\pi_{k+1}^m(\cdot| \sa))}\cdot  \left\|f_k^m(\sa, \cdot)\right\|_\infty \\
    &\le \sqrt{2\eta} \sqrt{\langle \pi_{k+1}^m(\cdot|\sa) - \pi_{k}^m(\cdot|\sa), f_k^m(\sa, \cdot) \rangle} \frac{1}{1-\gamma}.
\end{align*}
Thus we have $\left \langle \pi_{k+1}^m(\cdot|\sa) - \pi_{k}^m(\cdot|\sa), f_k^m(\sa, \cdot) \right \rangle \le \frac{2 \eta}{(1-\gamma)^2}$. Substituting it into Lemma~\ref{lem:c5} we have
\begin{align*}
    &\quad \sum_{k=1}^K \left \langle \pi(\cdot | \sa)- \pi_{k}(\cdot | \sa), f_k^m(\sa, \cdot) \right \rangle \\
    &\le \frac{2 \eta K}{(1-\gamma)^2} + \frac{\log{|\gA|}}{\eta}\\
    &\le \frac{2 \sqrt{2 \log{|\gA|} K} }{1-\gamma}
\end{align*}
where the last line is by setting $\eta = (1-\gamma)\sqrt{\frac{\log{|\gA|}}{2 K}}$.
\end{proof}

\begin{lemma} 
\label{lem:c7}
For any conditional policy $\pi :\gS \times \gA^{m-1} \xrightarrow{} \Delta(\gA)$,
\begin{align*}
    Q_\bpi^{1:m}(s_0, \bpi^{1:m}) \ge \min_{f \in \gF^m} \left( f(s_0, \bpi^{1:m}_k) + \lambda \gE^{1:m}(f, \bpi_k) \right) - \lambda \varepsilon_r.
\end{align*}
\end{lemma}

\begin{proof} For any conditional policy $\bpi^{1:m}$, and $\forall m \in \gN$. With \textbf{realizability} assumption we have $Q_\bpi^{1:m} = \argmin_f \sup_{\nu} \|f - \gT_{\bpi}^{1:m} f\|_{2, \nu}^2$ for any admissible $\nu$
\begin{align*}
    J(\bpi) &= Q_\bpi^{1:m}(s_0, \bpi^{1:m})\\
    &= Q_\bpi^{1:m}(s_0, \bpi^{1:m}) - \left( Q_\bpi^{1:m}(s_0, \bpi^{1:m}) + \lambda \gE^{1:m}(Q_\bpi^{1:m}, \bpi) \right) + \left( Q_\bpi^{1:m}(s_0, \bpi^{1:m}) + \lambda \gE^{1:m}(Q_\bpi^{1:m}, \bpi) \right)\\
    &\ge \min_{f \in \gF^m} \left( f(s_0, \bpi^{1:m}) + \lambda \gE^{1:m}(f, \bpi) \right) - \lambda \varepsilon_r,
\end{align*}
where in the last line we use Lemma~\ref{lem:c1}.
% where in the last line we use Lemma 1~\citep{xie2020q} and \textbf{realizability} assumption
% \begin{align*}
%     Q_\bpi^{1:m}(s, \bpi^{1:m}) - f_\bpi^{1:m}(s, \bpi^{1:m}) \ge - \frac{1}{1-\gamma} \left\|f_\bpi^{1:m} - \gT_\bpi^{1:m}f_\bpi^{1:m} \right\|_{2, d_\bpi} \ge 0.
% \end{align*}
\end{proof}

\paragraph{Proofs for Theorem~\ref{alg:2}.}
\begin{proof}
Use Lemma~\ref{lem:c7}, at the $k$-th iteration we have
\begin{align*}
    J(\bpi_k) &\ge \min_{f \in \gF^m} \left( f(s_0, \bpi_k^{1:m}) + \lambda \gE^{1:m}(f, \bpi_k) \right) - \lambda \varepsilon_r\\
    &\ge f_k^m(s_0, \bpi_k^{1:m}) - \lambda \varepsilon_r\\
    &= J_{\gM_k}(\bpi_k) - \lambda \varepsilon_r
\end{align*}
where $f_k^m$ is the multi-agent value function of Markov game $\gM_k$ (cf. Lemma~\ref{lem:c3}).

Therefore we have,
\begin{align*}
    J(\bpi_*) - J(\bar{\bpi}) &= \frac{1}{K} \sum_{k=1}^K (J(\bpi_*) - J(\bpi_k))\\
    &\le \frac{1}{K} \sum_{k=1}^K (J(\bpi_*) - J_{\gM_k}(\bpi_k)) + \lambda \varepsilon_r\\
    &\le \underbrace{\frac{1}{K} \sum_{k=1}^K (J_{\gM_k}(\bpi_*) - J_{\gM_k}(\bpi_k))}_{\textnormal{\uppercase\expandafter{\romannumeral1}}} + \underbrace{\frac{1}{K} \sum_{k=1}^K (J(\bpi_*) - J_{\gM_k}(\bpi_*))}_{\textnormal{\uppercase\expandafter{\romannumeral2}}} + \lambda \varepsilon_r
\end{align*}

\paragraph{Term \uppercase\expandafter{\romannumeral1}.} Apply performance difference lemma we have
\begin{align*}
    &\quad \frac{1}{K} \sum_{k=1}^K (J_{\gM_k}(\bpi_*) - J_{\gM_k}(\bpi_k))\\
    &= \frac{1}{1-\gamma} \frac{1}{K} \sum_{m=1}^N \E_{s \sim d_{\bpi_*}} \E_{\rva^{1:m-1}} \sum_{k=1}^K \left( Q_{\gM_k}^{\bpi_k}(\sa, \bpi_*^m) - Q_{\gM_k}^{\bpi_k}(\sa, \bpi_k^m) \right)\\
    &= \frac{1}{1-\gamma} \frac{1}{K} \sum_{m=1}^N \E_{s \sim d_{\bpi_*}, \rva^{1:m-1} \sim \bpi_*} \sum_{k=1}^K \left\langle \bpi_*^m(\cdot|\sa) - \bpi_k^m(\cdot|\sa), f_k^m(\sa, \cdot) \right\rangle
\end{align*}
where the last line is because $f_k^m$ is the multi-agent state-action value function for $\gM_k$. Then, if $\eta = (1-\gamma)\sqrt{\frac{\log{|\gA|}}{2 K}}$, Lemma~\ref{lem:c6} gives
\begin{align*}
    \textnormal{Term \uppercase\expandafter{\romannumeral1}} \le \frac{2N}{(1-\gamma)^2} \sqrt{\frac{2\log{|\gA|}}{K}}.
\end{align*}

\paragraph{Term \uppercase\expandafter{\romannumeral2}.} The following analysis holds for any $m \in \gN$ so we omit $m$ for clarity. For this term, again, we use Lemma 1~\citep{xie2020q} to transform it into norm over state-action distributions

\begin{align*}
J(\bpi_*) - J_{\gM_k}(\bpi_*)&=
    Q_{\bpi_*}(s, \bpi_*) -J_{\gM_k}(\bpi_*) \le  \frac{1}{1-\gamma} \left\|Q_{\bpi_*} - \gT_{\bpi_*, \gM_k }Q_{\bpi_*}\right\|_{2, d_{\bpi_*}}
\end{align*}

Note the definition of the auxiliary Markov game for which $f^k$ is its value function(cf. Lemma~\ref{lem:c3}) we have
\begin{align*}
    &\quad \frac{1}{1-\gamma} \left\|Q_{\bpi_*} - R_k - \gamma P_{\bpi_k}Q_{\bpi_*}\right\|_{2, d_{\bpi_*}} \quad R_k = (I - \gamma \gP) f_k \\
    &=  \frac{1}{1-\gamma} \left\|f_k - \gT_{\bpi_k}f_k\right\|_{2, d_{\bpi_*}}\\
    &\le \frac{\gC^{d_{\bpi_*}}_\mu}{1-\gamma} \left\|f_k - \gT_{\bpi_k}f_k\right\|_{2, \gD}
\end{align*}
which is no greater than $\frac{\gC^{d_{\bpi_*}}_\mu}{1-\gamma} \left(\sqrt{\varepsilon_r} + \sqrt{\frac{1}{\lambda(1-\gamma)}} \right)$, because $\gE(f_k, \bpi_k) \le \varepsilon_r + \frac{1}{(1-\gamma) \lambda}$: 
\begin{align*}
    f_k(s_0, \bpi_k) + \lambda \gE(f_k, \bpi_k) &= \min_{f} \left( f(s_0, \bpi_k) + \lambda \gE(f, \bpi_k) \right)\\
    &\le Q_{\bpi_k}(s_0, \bpi_k) + \lambda \gE(Q_{\bpi_k},\bpi_k), \quad \text{Lemma~\ref{lem:c2}}\\
    &\le \frac{1}{1-\gamma} + \lambda \varepsilon_r.
\end{align*}
The proof is completed by substituting $\lambda$ to the original expression.
\end{proof}

\section{ Pessimistic MA-PPO with General Function
Approximation}
\label{sec:pf-general}
In this section we extend the results from linear function approximation to general function approximation (cf. Section~\ref{sec:pess}).

\begin{algorithm}
\caption{Pessimistic Multi-Agent PPO with General Function Approximation}
\label{alg:4}
\begin{algorithmic}[1]
\REQUIRE Regularization coefficient $\lambda$.
\ENSURE Uniformly sample $k$ from $0, 1 \cdots K-1$, return $\bar{\bpi} = \bpi_k$.
\STATE Initialize uniformly: $\theta_0^m=0$ for every $ m \in \gN$.
\FOR{$k=0, 1,\dots, K-1$}
    \FOR{$m=1,2, \cdots, N$}
        \STATE Obtain the pessimistic estimate:\\
        $f_k^m \xleftarrow[]{}\argmin_{f \in \gF^m} \left(f(s_0, \bpi^{1:m}_k) + \lambda \gE^{1:m}(f, \bpi_k)  \right)$.
        \STATE Policy improvement: for any $(s, \rva^{1:m}) \in \gS \times \gA^{m}$,
        $$ \pi_{k+1}^m(a^m|\sa) \propto \pi_{k+1}^m(a^m|\sa) \cdot \exp{(\eta f_k^m(\sa, a^m))}.
        $$
    \ENDFOR
\ENDFOR
\end{algorithmic}
\end{algorithm}
In this setting, the value function is searched over a finite set $\gF^m$. We impose the following regularity conditions on the general function class
\begin{assumption}For any $m \in \gN$, $f \in \gF^m$ and $(s, \rva^{1:m}) \in \gS \times \gA^m$, it holds
\begin{gather}
    |f(\sa, a^m)| \le \frac{1}{1-\gamma},\\
    |\gF^m| \le |\gF|
\end{gather}
where $|\gF|$ is a certain positive number.
\end{assumption}

Instead of a pre-defined fixed policy class $\Pi$, now policy improvement is made 
upon $\gF^m$, formally, for $m \in \gN$ and $(\sa) \in \gS \times \gA^{m-1}$
\begin{align}
\label{eq:pim}
    \Pi^m = \left\{ \pi^m(\cdot| \sa) \propto \exp{\left( 
    \eta \sum_{j=0}^k f_j(\sa, \cdot) 
    \right): f_j \in \gF^m, 0 \le k \le K-1}
 \right\}.
\end{align}

Also note that under general function approximation, policy improvement has to be specific for each $(s, \rva^{1:m})$, which might become troublesome when the state space is enormous.

For general function approximation, two common expressivity assumptions on $\gF$ are required~\citep{antos2008learning,xie2021bellman}.

\begin{assumption}[Realizability]
For any $m \in \gN$, $\bpi \in \Pi^m$,
\begin{align*}
     \inf_{f \in \gF^m} \sup_{\text{admissable } \nu} \|f - \gT_{\bpi}^{1:m} f\|_{2, \nu}^2 = \zeta_{\gF}.
\end{align*}

\end{assumption}
where $\nu$ can be any admissible distribution over $\gS \times \gA^{1:m}$, and
\begin{assumption}[Completeness]
    For any $m \in \gN$, $\bpi \in \Pi^m$,
\begin{align*}
      \sup_{f \in \gF^m} \inf_{f \in \gF^m} \|f^\prime - \gT_{\bpi}^{1:m} f\|_{2, \mu}^2  = \zeta_{\gF}^\prime.
\end{align*}
\end{assumption}

We have the following concentration guarantees for general function approximation~\citep{xie2021bellman}.

\begin{lemma}
\label{lem:d1}
For any $m \in \gN$, $\bpi \in \Pi^m$, let 
$$
f_\bpi^{1:m} = \argmin_{f \in \gF^m} \sup_{\textnormal{admissable } \nu} \|f - \gT_{\bpi}^{1:m} f\|_{2, \nu}^2,
$$
with probability at least $1-\delta$, it holds
\begin{align}
    \label{eq:er}
    \gE^{1:m}(f_\bpi^{1:m}, \pi) & \le \gO \left( \frac{ \log{\frac{|\gF^m||\Pi^m|}{\delta}}}{n (1-\gamma)^{2}} + \zeta_{\gF}\right) \notag \\
    &= \gO \left( \frac{K  \log{\frac{|\gF^m|}{\delta}}}{n (1-\gamma)^{2}} + \zeta_{\gF}\right)= \varepsilon_r,
\end{align}

where we note that $|\Pi^m| \le |\gF^m|^K$ from~\eqref{eq:pim}.
\end{lemma}

\begin{lemma}
\label{lem:d2}
For any $m \in \gN$, $\bpi \in \Pi^m$, $f \in \gF^m$ , if $\gE^{1:m}(f, \bpi) \le \varepsilon$,  with probability at least $1-\delta$ it holds
$$
\left\| f - \gT_\bpi^{1:m} f \right\|_{2, \gD^m} \le  \gO \left( \frac{1}{1-\gamma} \sqrt{ \frac{K\log{ \frac{|\gF^m|}{\delta} }}{n}} \right) + \sqrt{\zeta_{\gF}^\prime }+ \sqrt{\zeta_{\gF}^\prime + \varepsilon}.
$$
\end{lemma} 

\begin{lemma} 
\label{lem:d6}
For any conditional policy $\pi :\gS \times \gA^{m-1} \xrightarrow{} \Delta(\gA)$,
\begin{align*}
    Q_\bpi^{1:m}(s_0, \bpi^{1:m}) \ge \min_{f \in \gF^m} \left( f(s_0, \bpi^{1:m}_k) + \lambda \gE^{1:m}(f, \bpi_k) \right) - \frac{\sqrt{\zeta_{\gF}}}{1-\gamma} - \lambda \varepsilon_r.
\end{align*}
\end{lemma}

\begin{proof} For any conditional policy $\bpi^{1:m}$, and $\forall m \in \gN$. Let $f_\bpi^{1:m} = \argmin_f \sup_{\nu} \|f - \gT_{\bpi}^{1:m} f\|_{2, \nu}^2$ for any admissible $\nu$
\begin{align*}
    J(\bpi) &= Q_\bpi^{1:m}(s_0, \bpi^{1:m})\\
    &= Q_\bpi^{1:m}(s_0, \bpi^{1:m}) - \left( f_\bpi^{1:m}(s_0, \bpi^{1:m}) + \lambda \gE^{1:m}(Q_\bpi^{1:m}, \bpi) \right) + \left( f_\bpi^{1:m}(s_0, \bpi^{1:m}) + \lambda \gE^{1:m}(f_\bpi^{1:m}, \bpi) \right)\\
    &\ge \min_{f \in \gF^m} \left( f(s_0, \bpi^{1:m}) + \lambda \gE^{1:m}(f, \bpi) \right) - \frac{\sqrt{\zeta_{\gF}}}{1-\gamma} - \lambda \varepsilon_r,
\end{align*}
where in the last line we use Lemma~\ref{lem:c1} and the realizability assumption
\begin{align*}
    f_\bpi^{1:m} - Q_\bpi^{1:m} \le \frac{\|f_\bpi^{1:m} - \gT_{\bpi}^{1:m} f_\bpi^{1:m}\|_{2, d_\bpi} }{1-\gamma} \le \frac{\sqrt{\zeta_{\gF}} }{1-\gamma}.
\end{align*}
% where in the last line we use Lemma 1~\citep{xie2020q} and \textbf{realizability} assumption
% \begin{align*}
%     Q_\bpi^{1:m}(s, \bpi^{1:m}) - f_\bpi^{1:m}(s, \bpi^{1:m}) \ge - \frac{1}{1-\gamma} \left\|f_\bpi^{1:m} - \gT_\bpi^{1:m}f_\bpi^{1:m} \right\|_{2, d_\bpi} \ge 0.
% \end{align*}
\end{proof}

Equipped with these useful lemmas, we are prepared to proceed to the main theorem for the general function approximation setting
\begin{theorem}
\label{thm:3}
Recall the definition of $\varepsilon_r$~\eqref{eq:er}, for the output policy $\bar{\bpi}$ attained by Algorithm~\ref{alg:4} in a fully cooperative Markov game, set $\eta = (1-\gamma)\sqrt{\frac{\log{|\gA|}}{2K}}$ and $\lambda = (1-\gamma)^{-1}\varepsilon_r^{-\frac{2}{3}}$. After $K$ iterations, w.p. at least $1-\delta$ we have
\begin{align*}
    J(\bpi_*) - J(\bar{\bpi}) \le \gO \left( \frac{N}{(1-\gamma)^2} \sqrt{\frac{\log{|\gA|}}{K}}
    + \frac{\gC^{d_{\bpi_*}}_\mu}{1-\gamma} \cdot \left( \frac{1}{1-\gamma} \sqrt[3]{\frac{K \log{\frac{|\gF|}{n}} }{n}}+  \sqrt{\zeta_\gF + \zeta_\gF^\prime} + \sqrt[3]{\zeta_\gF} \right)  \right)
\end{align*}
\end{theorem}

\begin{proof}
Use Lemma~\ref{lem:d6}, at the $k$-th iteration we have
\begin{align*}
    J(\bpi_k) &\ge \min_{f \in \gF^m} \left( f(s_0, \bpi_k^{1:m}) + \lambda \gE^{1:m}(f, \bpi_k) \right) - \frac{\sqrt{\zeta_{\gF}}}{1-\gamma} -\lambda \varepsilon_r\\
    &\ge f_k^m(s_0, \bpi_k^{1:m}) - \frac{\sqrt{\zeta_{\gF}}}{1-\gamma}- \lambda \varepsilon_r\\
    &= J_{\gM_k}(\bpi_k) - \frac{\sqrt{\zeta_{\gF}}}{1-\gamma} -\lambda \varepsilon_r
\end{align*}

Analogous to Appendix~\ref{sec:pf-pess}, we have
\begin{align*}
    J(\bpi_*) - J(\bar{\bpi}) &= \frac{1}{K} \sum_{k=1}^K (J(\bpi_*) - J(\bpi_k))\\
    &\le \frac{1}{K} \sum_{k=1}^K (J(\bpi_*) - J_{\gM_k}(\bpi_k)) + \frac{\sqrt{\zeta_{\gF}}}{1-\gamma}+ \lambda \varepsilon_r\\
    &\le \underbrace{\frac{1}{K} \sum_{k=1}^K (J_{\gM_k}(\bpi_*) - J_{\gM_k}(\bpi_k))}_{\textnormal{\uppercase\expandafter{\romannumeral1}}} + \underbrace{\frac{1}{K} \sum_{k=1}^K (J(\bpi_*) - J_{\gM_k}(\bpi_*))}_{\textnormal{\uppercase\expandafter{\romannumeral2}}} + \frac{\sqrt{\zeta_{\gF}}}{1-\gamma} + \lambda \varepsilon_r
\end{align*}

\paragraph{Term \uppercase\expandafter{\romannumeral1}.} The analysis for the optimization term is the same as that in Appendix~\ref{sec:pf-pess}. If $\eta = (1-\gamma) \sqrt{\frac{\log{|\gA|}}{2 K}}$, Lemma~\ref{lem:c6} gives
\begin{align*}
    \textnormal{Term \uppercase\expandafter{\romannumeral1}} \le \frac{2N}{(1-\gamma)^2} \sqrt{\frac{2\log{|\gA|}}{K}}.
\end{align*}

\paragraph{Term \uppercase\expandafter{\romannumeral2}.}Similar with Appendix~\ref{sec:pf-pess}

\begin{align*}
J(\bpi_*) - J_{\gM_k}(\bpi_*)&=
    Q_{\bpi_*}(s, \bpi_*) -J_{\gM_k}(\bpi_*)\\
    &\le \frac{\gC^{d_{\bpi_*}}_\mu}{1-\gamma} \left\|f_k - \gT_{\bpi_k}f_k\right\|_{2, \gD}
\end{align*}
which is no greater than 
$$
\frac{\gC^{d_{\bpi_*}}_\mu}{1-\gamma} \left(\gO \left( \frac{1}{1-\gamma} \sqrt{ \frac{K\log{ \frac{|\gF^m|}{\delta} }}{n}} \right) + \sqrt{\zeta_{\gF}^\prime }+ \sqrt{\zeta_{\gF}^\prime + \varepsilon_r} + \sqrt{\frac{1}{(1-\gamma)\lambda}} \right),
$$
because we have $\gE(f_k, \bpi_k) \le \varepsilon_r + \frac{1}{(1-\gamma)\lambda}$:
\begin{align*}
    f_k(s_0, \bpi_k) + \lambda \gE(f_k, \bpi_k) &= \min_{f} \left( f(s_0, \bpi_k) + \lambda \gE(f, \bpi_k) \right)\\
    &\le Q_{\bpi_k}(s_0, \bpi_k) + \lambda \gE(Q_{\bpi_k},\bpi_k), \quad \text{Lemma~\ref{lem:c2}}\\
    &\le \frac{1}{1-\gamma} + \lambda \varepsilon_r.
\end{align*}
The proof is completed by substituting $\lambda$ to the original expression.
\end{proof}

\section{Simulation}
\label{sec:sim}
In this section, we perform a toy example to showcase the superiority of our sequential update structure over naive independent policy gradient updates. We consider von Neumann’s ratio game, a simple stochastic game also used by ~\citep{daskalakis2020independent}.

In the game, there are only two agents, and each has an action space of 2. There is only one state, i.e., no state transition exists. The immediate reward for selecting actions $(a,b)$
 is $R(a,b)$
 the probability of stopping in each round is $S(a,b)$
. The value function $V(\pi_x, \pi_y)$
 for this game is given by 
 $$
 V = \frac{\pi_x^\top R \pi_y}{\pi_x^\top S \pi_y}.
 $$

The two agents cooperate with each other to maximize the value function. From now on, we shall use $(x,1-x)$ and $(y,1-y)$ 
 to represent both policies. We set parameters as
\begin{equation*}
R = 
\begin{bmatrix}
1 & 0.5 \\
-0.5 & 1
\end{bmatrix},\quad \text{and}\quad
R = 
\begin{bmatrix}
1 & 1 \\
0.1 & 0.1
\end{bmatrix}.
\end{equation*}

Consider the value function as a function of variables $x$ and $y$
, then the stationary point is near $(x,y) = (0.5,0)$
, at which the value function is $V \approx 0.46$
, which is smaller than the global maximum $V =1$.

To solve the problem, we adopt two algorithms: (1) our algorithm with sequential updates and (2) the independent (policy gradient) learning method. In both algorithms, we use softmax parametrization for policies. In particular, our log-linear parameterization (\ref{eq:loglinear}) becomes softmax parametrization by setting $\phi$
 as one-hot representations, i.e., for action $a$
, $\phi(a)^\top \theta =\theta_a$ where $\theta_a$
 represents the specific entry of $\theta$
 that corresponds to $a$ 
.

\paragraph{Results} We test our algorithm with sequential gradient updates and the independent learning method in different settings. The results are shown in Figure~\ref{fig:1}. \footnote{Implementation can be found at \url{https://github.com/zhaoyl18/ratio_game}.}
Below we discuss the empirical findings from this simulation study. 

First, we find that the independent policy optimization method often struggles around the stationary point (see (a)-(c)) that is not necessarily globally optimal. In this example, a big stepsize would help alleviate the issue (e.g., in (c), independent PG escapes the stationary point after 3000 iterations). However, the convergence to global optima is still slower than our method. We note that noise might help to escape the stationary points~\citep{jin2017escape}. Our findings align with the theoretical comparisons we made aforementioned. Even if the independent PG method is not trapped by a stationary point, from (d)-(f), our algorithm consistently outperforms in terms of maximizing the value function.

In this toy example, the optimization landscape is quite simple: only two agents participate, each with only two possible actions. No state transition is allowed, which is the main difficulty in performing sequential decision-making. We point out that, globally, there is only one stationary point. In such an effortless case, our algorithm consistently outperforms independent PG in mainly two folds. First, our algorithm does not struggle like independent PG when the current policy is near the stationary point where gradient information is few. Second, our algorithm demonstrates a fast convergence rate to the global maximum value function. Therefore, when the  complexity of the environment increases significantly, for instance: (1) multiple heterogeneous agents interact with each other and the unknown environment, (2) complex function approximators are adopted (e.g., deep neural networks), utilizing independent PG would be more problematic in terms of locating the global optimum because there will be more stationary points in the landscape.
  
Our findings showcase the necessity and usefulness of the conditional dependency structure, which helps us find a policy that enjoys a globally sub-optimal value function.

\begin{figure}[ht]
\begin{center}
\centerline{\includegraphics[width=0.9\columnwidth]{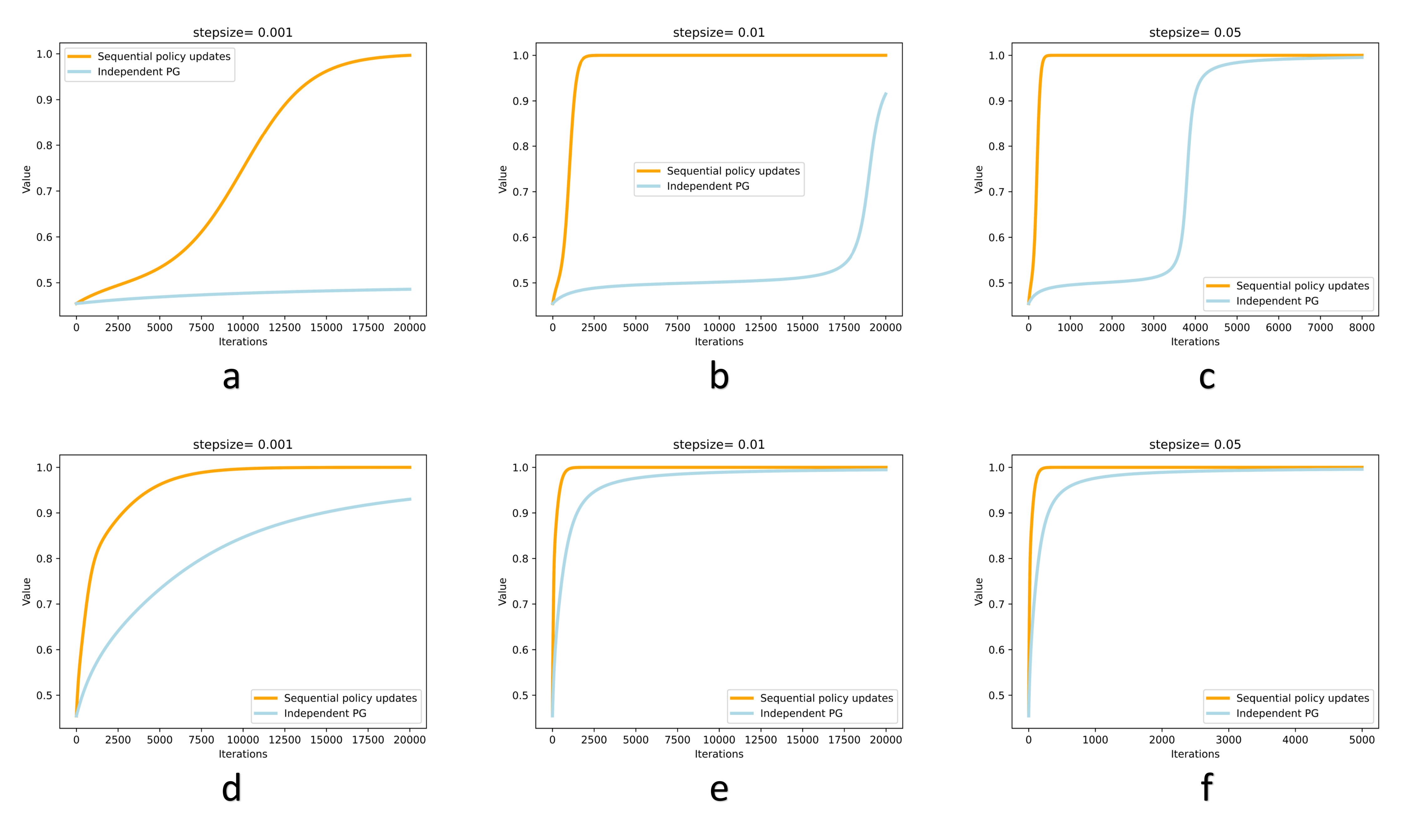}}
\caption{Performances of our algorithm and the independent learning method. In (a)-(c): policies are initialized close to the stationary point. In (d)-(f): both policies are uniformly initialized.
}
\label{fig:1}
\end{center}
\end{figure}

\end{document}